\documentclass{article}
\usepackage{newtxtext}

\usepackage[preprint]{neurips_2026}
\usepackage{siunitx}
\usepackage{wrapfig}
\usepackage{subcaption}
\usepackage{adjustbox}
\usepackage{natbib} 
\usepackage{amsmath, amsthm}
\usepackage{amssymb}
\usepackage{mathtools}
\usepackage{verbatim} 
\newtheorem{proposition}{Proposition}

\usepackage{xcolor}

\usepackage[utf8]{inputenc} 
\usepackage[T1]{fontenc}    
\usepackage{hyperref}       
\usepackage{url}            
\usepackage{booktabs}       
\usepackage{amsfonts}       
\usepackage{nicefrac}       
\usepackage{microtype}      
\usepackage{xcolor}         
\usepackage[normalem]{ulem}

\title{Hypernetworks for Dynamic Feature Selection}

\author{%
 Javier Fumanal-Idocin \\
  University of Essex\\
  Wivenhoe Park, Colchester, United Kingdom \\
  \texttt{j.fumanal-idocin@essex.ac.uk} \\
  \And
  Raquel Fernandez-Peralta \\
  Slovak Academy of Sciences \\
   Štefánikova 49, 814 73 Bratislava, Slovakia\\
\texttt{raquel.fernandez@mat.savba.sk}
  \AND
   Javier Andreu-Perez \\
  University of Essex\\
  Wivenhoe Park, Colchester, United Kingdom \\
  \texttt{j.andreu-perez@essex.ac.uk} \\
}

\begin{document}

\maketitle

\begin{abstract}
   Dynamic feature selection (DFS) is a machine learning framework in which features are acquired sequentially for individual samples under budget constraints. The exponential growth in the number of possible feature acquisition paths forces a DFS model to balance fitting specific scenarios against maintaining general performance, even when the feature space is moderate in size. In this paper, we study the structural limitations of existing DFS approaches to achieve an optimal solution. Then, we propose \textsc{Hyper-DFS}, a hypernetwork-based DFS approach that generates feature subset-specific classifier parameters on demand. We show that the use of hypernetworks compared to mask-embedding methods results in a smaller structural complexity bound. We also use a Set Transformer encoding to create a smooth conditioning space for the hypernetwork, so that functionally similar tasks are also geometrically close. In our benchmarks, \textsc{Hyper-DFS} performed best or second best compared to all state-of-the-art approaches on synthetic and real-life tabular data. It is also best or second best across all image datasets tested, and shows stronger zero-shot generalisation to feature subsets never seen during training than existing DFS approaches.
   
   Code available in: \url{https://github.com/Fuminides/hyper_DFS}
\end{abstract}

\section{Introduction}

Machine learning (ML) has achieved remarkable success in many real-life scenarios \citep{paleyes2022challenges}. However, the deployment of ML systems in budget-constrained environments still presents difficulties. For example, data is often incomplete and subject to acquisition costs that can vary dynamically with operational constraints \citep{erion2022cost}. Dynamic feature selection (DFS), also known as active feature learning, is a ML paradigm designed for such settings, in which a different set of features is chosen for each processed sample based on the information available \citep{karayev2013dynamic}. Because of this, the model can be adapted to different conditions where some features are more relevant to a particular case, or where economic and time constraints prevent access to some information.

Existing approaches to DFS are broadly categorised into three paradigms: generative models, reinforcement learning-based policies and greedy acquisition methods. Generative models learn the conditional distribution of the missing features \citep{ma2019eddi}. Reinforcement learning-based approaches frame the sequential feature acquisition as a Markov Decision Process, learning a policy that maximises a long-term reward signal \citep{janisch2019classification}. Greedy methods query at each step the feature that optimises an immediate utility metric, with Conditional Mutual Information (CMI) emerging as a particularly popular criterion \citep{gadgil2024estimating}, although some concerns have been raised about the validity of this metric \citep{norcliffe2025stochastic}. These approaches generally use a single shared predictor for all possible observation statuses, as having a dedicated predictor for each is impossible due to their exponential number. Approximated per-subset predictors have been proposed in the literature \citep{valancius2024acquisition}, but they incur significant computational and memory overhead that prevents scaling to large feature spaces.

Although the value proposition of DFS is clear, existing approaches present several limitations: lack of clear superiority compared to static selection models \citep{erion2022cost}; the DFS models' decisions can be unreliable during the sequential acquisition phase \citep{fumanal2025dynamic}; and tendency to overfit to small subsets of features, especially relevant if the acquisition policy is heavily biased towards selecting the same features for every sample \citep{liu2025exploring}. 

In this paper, we address these issues by studying the conceptual limitations of existing DFS models. More precisely, we study how feasible it is for a DFS classifier to return optimal predictions for every possible feature subset. We study the difference between the ideal objective and the tractable training objective used in the literature to train DFS predictors. We also argue that mask-concatenation approaches, in which the observed subset $S$ is encoded as a conditioning signal to a shared neural network, are instances of \emph{embedding methods} as described in \citet{galanti2020modularity}, which means that they incur a complexity requirement that scales with both the input dimension and the task-conditioning dimension simultaneously, rather than with each independently. Then, we propose \textsc{Hyper-DFS}, a hypernetwork-based framework that generates subset-specific classifier parameters on demand using a Set Transformer to encode the task into a continuous representation. We show how to train the DFS policy for \textsc{Hyper-DFS} and benchmark it against other previous single and multi-model DFS approaches.

In summary, our contributions are the following:

\begin{itemize}
    \item We study the fundamental limitations of existing DFS approaches, and show how tractable objectives diverge from the ideal DFS solution and how a hypernetwork can close this gap.

    \item We propose \textsc{Hyper-DFS}, a hypernetwork-based framework that generates subset-specific classifiers on demand. We use Set Transformers to map discrete feature masks into a continuous representation that promotes smooth interpolation over different tasks, which also helps generalisation to feature subsets unseen during training.

    \item We identify and address two hypernetwork training issues particularly relevant in the DFS setting: variable subset cardinality causes variance drift in the generated weights, which we solve using a feature absence embedding vector; and multiple conditioning tasks make the hypernetwork training unstable, which we solve using a controlled distribution of masks per batch during training. 
    
     \item We evaluate \textsc{Hyper-DFS} on a broad benchmark suite spanning synthetic and real data, in tabular and image domains, demonstrating consistent improvements over single-model and multi-model baselines.
\end{itemize}

\section{Preliminaries}
\paragraph{Notation.}
Let $\mathbf{x} = (x_1, \ldots, x_M) \in \mathcal{X} \subseteq \mathbb{R}^M$ denote a sample with $M$ features, $y \in \mathcal{Y}$ its label, $S \subseteq [M] = \{1, \ldots, M\}$ a subset of observed features and $\bar{S} = [M] \setminus S$ for its complement. We denote as $\mathbf{I}_S \in \{0,1\}^{M}$ the binary mask that marks which features are present in $S$, which we also call the \emph{knowledge status}. We use $\mathbf{x}_S = (x_j)_{j \in S} \in \mathbb{R}^{|S|}$ to denote the subvector of observed entries. We use $S_t$ to indicate the feature set observed in the moment $t$ of the sequential feature acquisition process. 
\subsection{Dynamic Feature Selection}

In standard supervised learning, a fixed set of features is assumed to be available for every sample at both training and test time. DFS relaxes this assumption by allowing features to be acquired sequentially, on a per-instance basis and under budget constraints \citep{karayev2013dynamic}. A DFS system consists of two interacting components: a \emph{selector} and a \emph{predictor}. The selector $q_\psi\colon \mathbb{R}^{M} \times \{0,1\}^{M} \to \{1,\dots,M\}$ selects the feature to query. The predictor $f_\theta \colon \mathbb{R}^{M} \times \{0,1\}^{M} \to \mathcal{Y}$ produces a prediction from the observed features and knowledge status. Since accommodating variable-sized inputs in the predictor can be impractical, a common approach is the \emph{impute-then-predict} paradigm \citep{lemorvan2020}, which consists of imputing the missing features $\bar{S}$ before prediction. The imputation is typically done with the mean, although more sophisticated generative schemes have also been studied \citep{ma2019eddi}; in non-linear problems, however, the choice of imputation has limited impact on predictive performance, as model expressivity largely compensates for imputation errors \citep{lemorvan2025}. We denote by $\tilde{\mathbf{x}}_S$ the vector $\mathbf{x}_S$ with values in $\bar{S}$ replaced by their means. The predictor is then trained to minimise the expected loss over feature subsets:
\begin{equation}
    \theta^* = \arg \min_\theta \,\mathbb{E}_{S}\,\mathbb{E}_{x,y}\,
    \mathcal{L}(f_\theta(\tilde{\mathbf{x}}_S, \mathbf{I}_S), y),
    \label{eq:single_model_obj}
\end{equation}
where the distribution of the feature sets can be given by a uniform distribution, or by the selector. The selector is typically trained jointly with the predictor, either through reinforcement learning~\citep{janisch2019classification, yoon2018invase}, greedy maximisation of an information-theoretic acquisition criterion, like CMI~\citep{covert2023learning, gadgil2024estimating, chattopadhyayvariational}, or hand-crafted scoring functions~\citep{norcliffe2025stochastic}.

\subsection{Hypernetworks}
The idea of using one neural network to generate the weights of another has roots in early work on fast weight programmers \citep{schmidhuber1992learning}, where a ``slow'' network learned to produce context-dependent weight changes for a ``fast'' network. This concept was formalised and scaled to modern deep learning architectures by \citet{ha2017hypernetworks}, who introduced the term \emph{hypernetwork} to describe a network that takes a conditioning input and outputs the parameters of a separate \emph{primary} network.

Formally, let $g_\phi$ denote the hypernetwork with parameters $\phi$, which produces the weights $\theta_z = g_\phi(z)$ of the primary network $f_\theta$ conditioned on an input $z$. The output for a data point $x$ is then $\hat{y} = f_{g_\phi(z)}(x)$, and the hypernetwork parameters are optimised end-to-end:
\begin{equation}
\min_{\phi}\ \mathbb{E}_{x,y,z}\left[\mathcal{L}\left( f_{g_\phi(z)}(x),\, y \right)\right].
\end{equation}
 Because the loss is backpropagated through both $f_\theta$ and $g_\phi$, the hypernetwork learns weight-generation strategies directly optimised for the downstream task. A key property of hypernetworks is that they encode an entire family of related models within a single set of parameters $\phi$, achieving a form of soft weight sharing across conditioning inputs \citep{chauhan2024hypernetworks}.

\section{Structural Risks in Dynamic Feature Selection}
\label{sec:structural_risk}

\subsection{Compromises in DFS predictors}

The ideal DFS predictor would require a Bayes-optimal parameter set, $\theta^*_S$, for every subset, solving:
\begin{equation}\label{eq:ideal_obj}
   \theta^*_S = \arg \min_{\theta_S}
\,\mathbb{E}_{x,y}\,\mathcal{L}(f_{\theta_S}(\mathbf{x}_S), y),
\qquad \forall S \subseteq [M].
\end{equation}
However, this is not generally feasible in practice as it would require $2^M$ different parametrizations. Because of that, prior work in DFS has mostly focused on one shared parametrization that amortises the cost of learning the $2^M$ individual predictors. This shared predictor is normally trained first using Eq.~\eqref{eq:single_model_obj} as target loss with a uniform distribution over $S$, so that it performs well on average on any feature subset. Then, it is trained using feature sets sampled from the selector $q_\psi$, so that: $S_{t+1} = S_t \cup \{i_t\}$ with $i_t \sim q_\psi(\tilde{\mathbf{x}}_{S_t}, \mathbf{I}_{S_t})$. This two-step solution is designed to mitigate the problems that arise because the number of possible knowledge statuses is $2^M$.  The first phase encourages $f_\theta$ to perform well on arbitrary subsets; the second phase concentrates the gradient signal on the subsets that $q_\psi$ actually selects, which typically form a much smaller and more structured region of $2^{[M]}$. 

This amortisation is what makes the DFS problem tractable, but it comes at a cost. Because a single $\theta$ must amortise the cost of such a big task space simultaneously, it depends on how well the Bayes-optimal predictors can be represented under the same shared structure. When this can be done, the compromise is mild. Otherwise, the shared $\theta$ is optimal on average but not necessarily close to optimal on any individual subset.
Conflicts can also happen at the gradient level, so that updates beneficial for one subset actively degrade performance on another \citep{liu2021cagrad}. Moreover, subsets that $q_\psi$ rarely or never visits receive little or no gradient signal at all, so the predictor may systematically underperform on them, despite the uniform-subset pre-training phase \citep{fumanal2025dynamic}. So, in practical settings, we can assume that there will be some degree of compromise for some feature sets \citep{muller2025pattern, lobo2025primer}.

\subsection{From Ideal to Tractable: Hypernetworks}
\label{subsec:hypernet_objective}

If we replace the discrete lookup table $\{\theta_S\}_{S \subseteq [M]}$ in the ideal objective (Eq.~\eqref{eq:ideal_obj}) with a function $g_\phi \colon \{0,1\}^M \to \Theta$ that generates model parameters from the knowledge status $\mathbf{I}_S$:
\begin{equation}
    \phi^* = \arg \min_\phi \,\mathbb{E}_{S}\,
    \mathbb{E}_{x,y}\,\mathcal{L}\bigl(f_{g_\phi(\mathbf{I}_S)}(\tilde{\mathbf{x}}_S),\, y\bigr),
    \label{eq:hypernet_obj}
\end{equation}
we naturally obtain a hypernetwork objective \citep{ha2017hypernetworks}. Variable $g_\phi$ generates the parameters of the primary network $f$, which then produces predictions from the input $\mathbf{\tilde{x}}_S$. Compared with the optimal solution (Eq.~\eqref{eq:ideal_obj}), the cost of learning and storing an exponential set of parameters is replaced by the cost of learning a single network $g_\phi$. Compared with the tractable solution (Eq.~\eqref{eq:single_model_obj}), the predictor parameters are now a function of $\mathbf{I}_S$, so we are not forced to find a single primary network to fit all subsets. 
\citet{galanti2020modularity} formalise the structural advantage of hypernetworks over mask-conditioned predictors. A mask-conditioned predictor is an instance of an \emph{embedding method}: the subset $S$ enters the network as a conditioning signal that is fused with the data features. For data dimension $m_1$, task dimension $m_2$, and target-function smoothness parameter $r$ (in the DFS setting, $m_1 = M$ and $m_2 = M$ for binary-mask conditioning), Theorem~4 in \citet{galanti2020modularity} shows that a hypernetwork primary network of complexity $\mathcal{O}(\varepsilon^{-m_1/r})$ is sufficient to achieve error $\leq \varepsilon$, while their Theorems~2 and 3 show that the primary network of any embedding method must have complexity at least $\Omega(\varepsilon^{-(m_1+m_2)})$.

Prior work also provides evidence that hypernetwork training with gradient descent has been successfully applied for amortised inference across large task families~\citep{zhao2020meta, beck2023hypernetworks, voncontinual}. However, the amortisation advantage also depends on the geometry of the conditioning space because $g_\phi$ is a neural network with bounded Lipschitz constant~\citep{bartlett2017spectrally}. This means that it cannot produce very different primary networks for nearby conditioning inputs. So, if the optimal parameters $\theta^*_S$ change abruptly between similar subsets, the hypernetwork is forced into a compromise. We mitigate this in our proposal by encoding $S$ through a 
Set Transformer (Section~\ref{subsec:subset_encoding}). This produces a continuous representation of the knowledge state, in which feature subsets that induce similar Bayes-optimal predictors map to nearby points. In this way, the 
Lipschitz constraint helps generalisation rather than becoming a limitation.

\section{Proposed Method}
In this section, we introduce a new DFS approach: \textsc{Hyper-DFS}. Unlike prior work that relies on a conditioning vector to model the observation patterns over the same predictor model, \textsc{Hyper-DFS} generates subset-specific predictors on demand via a hypernetwork.
To model the observation status, we use a Set Transformer representation \citep{lee2019set}, which maps discrete observation masks into a smooth, functional-aware geometry. Finally, to maintain the size of the hypernetwork small, \textsc{Hyper-DFS} uses a neural network to reduce the size of the primary network input, called the \emph{compressor}, which is common practice in the hypernetworks literature~\citep{chauhan2024hypernetworks}.

\subsection{Smooth Feature Set Encoding for \textsc{Hyper-DFS}}
\label{subsec:subset_encoding}

A key design question for \textsc{Hyper-DFS} is how to represent the knowledge status as a conditioning signal for the hypernetwork $g_\phi$. Prior work uses binary masks. We argue this is a poor choice because binary masks treat all feature positions as interchangeable. This means that the encoding has no way to represent that some features are semantically related. So, two subsets that differ only in one semantically similar feature receive representations as distant as two subsets that differ in one entirely unrelated feature.

We instead propose encoding the observed feature subset using a Set Transformer encoder built from Induced Set Attention Blocks (ISAB)~\citep{lee2019set}. Let $\mathbf{e}^0_i, \mathbf{e}^1_i \in \mathbb{R}^d$ denote two learned embeddings for feature $i$, representing its absent and present states respectively, where $d$ is the token-embedding dimension. Given the knowledge status $\mathbf{I}_S \in \{0,1\}^M$, we form the input matrix $\mathbf{T}(\mathbf{I}_S) \in \mathbb{R}^{M \times d}$ whose $i$-th row is $\mathbf{t}_i = I_{S,i}\,\mathbf{e}^1_i+\bigl(1 - I_{S,i}\bigr)\,\mathbf{e}^0_i$, and apply the Set Transformer $f_\omega \colon \mathbb{R}^{M \times d} \to \mathbb{R}^{M \times d}$ with parameters $\omega$, followed by sum-pooling along the token dimension and a neural output transformation $f_\rho \colon \mathbb{R}^d \to \mathbb{R}^{d'}$ with parameters $\rho$ that maps the pooled representation to the conditioning vector of dimension $d'$:
\begin{equation} \label{eq:subset_encoding}
    \mathbf{z}_{\omega,\rho}(\mathbf{I}_S) = f_\rho\!\left(
        \sum_{i=1}^{M} \bigl[f_\omega(\mathbf{T}(\mathbf{I}_S))\bigr]_i
    \right),
    \qquad
    \tilde{\mathbf{z}}_{\omega,\rho}(\mathbf{I}_S) =
    \frac{\mathbf{z}_{\omega,\rho}(\mathbf{I}_S)}
         {\|\mathbf{z}_{\omega,\rho}(\mathbf{I}_S)\|_2},
\end{equation}
where $\bigl[f_\omega(\mathbf{T}(\mathbf{I}_S))\bigr]_i \in \mathbb{R}^d$ denotes the $i$-th output token of the encoder and $\tilde{\mathbf{z}}_{\omega,\rho}(\mathbf{I}_S) \in \mathbb{R}^{d'}$ is the conditioning vector passed to the hypernetwork. Both $d$ and $d'$ are hyperparameters of the architecture. The self-attention mechanism allows each feature's contribution to be modulated by the other features present, capturing inter-feature dependencies in the conditioning signal. Using distinct embeddings for absent and present states also prevents a systematic bias that would arise with presence-only embeddings: if the pooled representation depended only on $\sum_{i \in S} \mathbf{e}^1_i$, its norm would scale with $|S|$, biasing the normalisation statistics and gradient magnitudes towards the observed subset size. We discuss this in detail in Appendix~\ref{apx:weight_variance}. The parameters $(\omega, \rho)$ are trained end-to-end with the hypernetwork and primary network.

\subsection{Training the \textsc{Hyper-DFS}}
\label{sec:grad_dilution}

\paragraph{Training the predictor.} The predictor is trained using a variant of Eq.~\eqref{eq:hypernet_obj} with the encoding defined in Eq.~\eqref{eq:subset_encoding} and a compressor $c_{\eta}\colon\mathbb{R}^{M}\to\mathbb{R}^{M'}$, where $M'$ is a hyperparameter:
\begin{equation}
	\arg \min_{\phi, \omega,\rho,\eta} \mathbb{E}_{S}\,
	\mathbb{E}_{x,y}\,\mathcal{L}(f_{g_\phi(\tilde{\mathbf{z}}_{\omega, \rho}(\mathbf{I}_S))}(c_\eta (\mathbf{\tilde{x}}_S)),\, y).
	\label{eq:hyperdfs_ob_obj}
\end{equation}
However, this training objective has two problems. First, mixing many different knowledge statuses in the same batch can be problematic, especially at the beginning of training, as the gradient may cancel if the contributions from different $S$ point in opposing directions~\citep{yu2020gradient}. Secondly, the compressor receives a backward signal that is itself modulated by the generated, mask-dependent primary network, which at initialisation behaves as a collection of random projection heads. This is the setting in which shared representations are known to be most prone to collapse~\citep{chen2021exploring}. We did not empirically observe this to be a problem in most cases, but it can be problematic in specific datasets. 

To mitigate the issue, we introduce two complementary modifications. First, we introduce a pre-training phase in which each batch uses a fixed small number of knowledge status vectors per batch, so that the batch uses a single primary network. This removes possible within-batch cancelling $S$-specific gradient directions. Second, we apply a linear learning-rate warm-up over the first $T_\mathrm{warm}$ steps, which allows the hypernetwork and the compressor to co-adapt before the optimiser takes large steps and the hypernetwork commits to a particular mapping. The problem is described in more detail in Appendix \ref{apx:grad_dilution}. The empirical study that shows the validity of the proposed solution is in Appendix \ref{app:ablations}.
 This pre-training phase produces a warm-started predictor that is subsequently fine-tuned jointly with the selector, as described in the next paragraph.

\paragraph{Training the selector.} The selector is trained similarly to other approaches in the literature \citep{chattopadhyayvariational}. It consists of a small neural network $q_\psi$ that takes as input the observed features and the knowledge status and produces a score for each candidate feature, then returns the feature with the highest score among the unobserved features. To train it, for each training sample $(\mathbf{x}, y)$, we draw a trajectory length $\tau \sim \mathrm{Uniform}(\{1, \ldots, B_{\max}\})$, where $B_{\max}$ is the maximum acquisition budget the model is expected to handle at test time, and let the selector pick $\tau$ features one at a time. At each step $t$, the selector produces scores $\mathbf{s}_\psi(\tilde{\mathbf{x}}_{S_t}, \mathbf{I}_{S_t}) \in \mathbb{R}^{M}$, with scores of already-observed features set to $-\infty$ so they cannot be selected again. The chosen feature $i_t$ is then used to update the knowledge status, $S_{t+1} = S_t \cup \{i_t\}$. Because $\arg\max$ is not differentiable, we replace it with the Gumbel-softmax with straight-through estimator~\citep{jang2017categorical, maddison2017concrete} during training, which lets gradients flow through each selection step. After $\tau$ steps, we apply the predictor to obtain $\hat{y} = f_{g_\phi(\tilde{\mathbf{z}}_{\omega,\rho}(\mathbf{I}_{S_\tau}))}(c_\eta(\tilde{\mathbf{x}}_{S_\tau}))$ and compute the cross-entropy loss $\mathcal{L}(\hat{y}, y)$, which is backpropagated through the full trajectory to update the selector and predictor parameters jointly. At test time, the Gumbel-softmax is no longer needed, and the selector simply picks the highest-scoring unobserved feature at each step, recovering the deterministic $q_\psi$ defined above.


\section{Experimentation}

\subsection{Datasets and Methods}

\paragraph{Datasets.} We evaluate \textsc{Hyper-DFS} on a benchmark suite spanning synthetic, real-world tabular, and image domains. \textit{Synthetic} datasets: the three binary tasks of~\citet{yoon2018invase} and  \textbf{Cube}~\citep{kompella2016optimal,zannone2019odin}, an 8-class, 20-feature task, are included as standard benchmarks from the DFS literature. We additionally introduce two diagnostic benchmarks where subset-adaptive specialisation is necessary to obtain good performance: \textbf{Synergistic Pairs}, where individual features carry zero marginal information, and \textbf{Proxy Substitution}, where the Bayes-optimal predictor changes with the observed subset (see Appendix~\ref{apx:datasets_details} for details about these two new datasets). \textit{Tabular:} nine datasets spanning medical, financial, genetic, and scientific domains: \textbf{Diabetes}~\citep{strack2014impact}, \textbf{Heart Disease}~\citep{Detrano1989InternationalAO},  \textbf{Cirrhosis}~\citep{fleming2013counting}, \textbf{Wine}~\citep{wine_109}, \textbf{Yeast}~\citep{yeast_110}, \textbf{Bank Marketing}~\citep{moro2014data}, \textbf{California housing}~\citep{pace1997sparse}, \textbf{MiniBooNE}~\citep{roe2005boosted}, \textbf{METABRIC}~\citep{curtis2012genomic}. Ranging from a few hundred to over $10^5$ samples and 8 to 47 features. \textit{Image:} \textbf{MNIST}~\citep{lecun2002}, \textbf{Fashion MNIST}~\citep{xiao2017fashion}, \textbf{SVHN}~\citep{netzer2011svhn}, \textbf{Imagenette}~\citep{howard2019}, a 10-class ImageNet subset, all evaluated under a patch-based sequential acquisition protocol, and \textbf{PCam}~\citep{veeling2018pcam}, a histopathology benchmark in which metastatic tissue must be localised within a lymph node patch.

\paragraph{\textsc{Hyper-DFS} approaches.} We evaluate a single-model and an ensembling variant. \textsc{Hyper-DFS} performs a single forward pass through the hypernetwork, producing one primary network per subset. \textbf{HyperEns} is a conventional ensemble of $K{=}5$ independently trained hypernetworks sharing the same selector. 

\paragraph{Baselines.}
We compare against eight state-of-the-art DFS methods representing the main paradigms in the
literature. \textbf{DIME}~\citep{gadgil2024estimating} and \textbf{VIP}~\citep{chattopadhyayvariational} are
greedy acquisition methods: DIME maximises a learned estimate of CMI, whilst VIP selects features via a variational approximation to information gain.
\textbf{CWCF}~\citep{janisch2019classification} and \textbf{INVASE}~\citep{yoon2018invase} are reinforcement learning approaches: CWCF trains a Q-learning agent for cost-sensitive sequential acquisition, whilst INVASE uses an actor-critic network for instance-wise variable selection. \textbf{SEFA}~\citep{norcliffe2025stochastic} uses stochastic latent encodings to account for unobserved feature values during acquisition. \textbf{RePa}~\citep{fumanal2025dynamic} learns weight transforms to adjust the classifier to different feature subsets. \textbf{EDDI}~\citep{ma2019eddi} uses a variational autoencoder to impute missing values. For multi-model approaches, we include \textbf{AACO}~\citep{valancius2024acquisition}, which is a non-greedy method that maintains a dictionary of subset-specific classifiers and uses nearest-neighbour conditionals to guide sequential acquisition. We also construct three multi-model baselines based on VIP acquisition policy, using the same neural architecture as in the other baselines for the predictors: \textbf{Ensemble} trains an ensemble of $K=5$ independent predictors and aggregates predictions by averaging; \textbf{Card.} partitions the parameter budget into 2, each trained on subsets within a specific cardinality range and used for routing at inference time. The cardinality range boundaries are chosen to distribute the cardinality range of the feature subsets evenly; \textbf{MoE} uses a learned routing function that takes as input the knowledge status and weights the decision of $K=5$ models.

\paragraph{Evaluation Protocol.}
Experiments use a uniform feature cost setting, reporting average F1 score across all budget levels from $2$ to $10$ acquired features. All methods are evaluated using stratified 5-fold cross-validation and we report the mean and standard deviation across folds. To test feature subset generalisation, we evaluate the DFS methods on random feature subsets withheld entirely from training. Ablations and further results are detailed in Appendix~\ref{apx:additional_results}. Full implementation details, hyperparameter search ranges, and reproducibility information are provided in Appendix~\ref{apx:implementation}.

\subsection{Empirical Results and Benchmarking}

\begin{table}[t]
	\centering
	\caption{AUAC-F1 on synthetic benchmarks (mean over folds, \%).}
	\label{tab:synthetic_auac_f1}
	\adjustbox{width=\linewidth}{
		\begin{tabular}{
				l
				*{7}{S[table-format=2.2, detect-weight=true]}
				|
				*{4}{S[table-format=2.2, detect-weight=true]}
				|
				*{2}{S[table-format=2.2, detect-weight=true]}
			}
			\toprule
			& \multicolumn{7}{c}{Single-model} & \multicolumn{4}{c}{Multi-model} & \multicolumn{2}{c}{Hypernetworks} \\
			\cmidrule(lr){2-8}\cmidrule(lr){9-12}\cmidrule(lr){13-14}
			{Dataset} & {CWCF} & {DIME} & {EDDI} & {INVASE} & {RePa} & {SEFA} & {VIP} & {AACO} & {Card.} & {Ensemble} & {MoE} & {HyperEns} & {Hyper-DFS} \\
			\midrule
			Cube & 29.81 & 31.97 & 18.20 & 22.99 & \bfseries 46.19 & 43.05 & 41.20 & 38.70 & 35.00 & \bfseries 40.40 & 39.15 & {\bfseries\underline{47.47}} & 41.42 \\
			{} & {\scriptsize\textcolor{gray}{$\pm$2.88}} & {\scriptsize\textcolor{gray}{$\pm$0.79}} & {\scriptsize\textcolor{gray}{$\pm$2.04}} & {\scriptsize\textcolor{gray}{$\pm$1.79}} & {\scriptsize\textcolor{gray}{$\pm$0.51}} & {\scriptsize\textcolor{gray}{$\pm$0.82}} & {\scriptsize\textcolor{gray}{$\pm$1.14}} & {\scriptsize\textcolor{gray}{$\pm$0.94}} & {\scriptsize\textcolor{gray}{$\pm$0.89}} & {\scriptsize\textcolor{gray}{$\pm$1.76}} & {\scriptsize\textcolor{gray}{$\pm$1.93}} & {\scriptsize\textcolor{gray}{$\pm$1.18}} & {\scriptsize\textcolor{gray}{$\pm$0.85}} \\
			Sim1 & 69.60 & 67.47 & 65.15 & 67.76 & 75.97 & 74.99 & \bfseries 75.99 & \bfseries 74.80 & 57.77 & 72.05 & 72.70 & {\bfseries\underline{78.19}} & 75.87 \\
			{} & {\scriptsize\textcolor{gray}{$\pm$3.02}} & {\scriptsize\textcolor{gray}{$\pm$1.07}} & {\scriptsize\textcolor{gray}{$\pm$3.02}} & {\scriptsize\textcolor{gray}{$\pm$1.49}} & {\scriptsize\textcolor{gray}{$\pm$0.71}} & {\scriptsize\textcolor{gray}{$\pm$0.96}} & {\scriptsize\textcolor{gray}{$\pm$0.56}} & {\scriptsize\textcolor{gray}{$\pm$0.83}} & {\scriptsize\textcolor{gray}{$\pm$0.83}} & {\scriptsize\textcolor{gray}{$\pm$3.66}} & {\scriptsize\textcolor{gray}{$\pm$2.57}} & {\scriptsize\textcolor{gray}{$\pm$0.71}} & {\scriptsize\textcolor{gray}{$\pm$0.65}} \\
			Sim2 & 56.84 & 55.67 & 48.68 & 53.76 & 67.71 & 65.99 & {\bfseries\underline{68.10}} & \bfseries 64.80 & 53.81 & 62.20 & 60.54 & \bfseries 67.14 & 66.53 \\
			{} & {\scriptsize\textcolor{gray}{$\pm$4.07}} & {\scriptsize\textcolor{gray}{$\pm$2.72}} & {\scriptsize\textcolor{gray}{$\pm$2.17}} & {\scriptsize\textcolor{gray}{$\pm$2.43}} & {\scriptsize\textcolor{gray}{$\pm$0.17}} & {\scriptsize\textcolor{gray}{$\pm$0.88}} & {\scriptsize\textcolor{gray}{$\pm$0.67}} & {\scriptsize\textcolor{gray}{$\pm$0.78}} & {\scriptsize\textcolor{gray}{$\pm$4.17}} & {\scriptsize\textcolor{gray}{$\pm$3.59}} & {\scriptsize\textcolor{gray}{$\pm$1.77}} & {\scriptsize\textcolor{gray}{$\pm$0.70}} & {\scriptsize\textcolor{gray}{$\pm$0.66}} \\
			Sim3 & 57.43 & 56.91 & 38.97 & 51.67 & {\bfseries\underline{68.68}} & 63.68 & 62.95 & \bfseries 63.82 & 50.27 & 55.58 & 60.01 & \bfseries 67.22 & 66.51 \\
			{} & {\scriptsize\textcolor{gray}{$\pm$3.45}} & {\scriptsize\textcolor{gray}{$\pm$3.08}} & {\scriptsize\textcolor{gray}{$\pm$0.01}} & {\scriptsize\textcolor{gray}{$\pm$1.75}} & {\scriptsize\textcolor{gray}{$\pm$0.37}} & {\scriptsize\textcolor{gray}{$\pm$0.77}} & {\scriptsize\textcolor{gray}{$\pm$0.45}} & {\scriptsize\textcolor{gray}{$\pm$0.66}} & {\scriptsize\textcolor{gray}{$\pm$4.31}} & {\scriptsize\textcolor{gray}{$\pm$2.84}} & {\scriptsize\textcolor{gray}{$\pm$3.69}} & {\scriptsize\textcolor{gray}{$\pm$0.80}} & {\scriptsize\textcolor{gray}{$\pm$0.87}} \\
			ProxySub & 82.28 & 92.78 & 86.72 & 92.31 & \bfseries 96.77 & 96.28 & 96.61 & \bfseries 96.62 & 80.80 & 87.53 & 84.81 & {\bfseries\underline{96.86}} & 96.73 \\
			{} & {\scriptsize\textcolor{gray}{$\pm$7.43}} & {\scriptsize\textcolor{gray}{$\pm$3.39}} & {\scriptsize\textcolor{gray}{$\pm$0.19}} & {\scriptsize\textcolor{gray}{$\pm$0.87}} & {\scriptsize\textcolor{gray}{$\pm$0.28}} & {\scriptsize\textcolor{gray}{$\pm$0.25}} & {\scriptsize\textcolor{gray}{$\pm$0.36}} & {\scriptsize\textcolor{gray}{$\pm$0.36}} & {\scriptsize\textcolor{gray}{$\pm$5.98}} & {\scriptsize\textcolor{gray}{$\pm$5.22}} & {\scriptsize\textcolor{gray}{$\pm$8.54}} & {\scriptsize\textcolor{gray}{$\pm$0.26}} & {\scriptsize\textcolor{gray}{$\pm$0.25}} \\
			SynPairs & 52.43 & 53.28 & 48.48 & 49.69 & 51.53 & 63.29 & \bfseries 65.26 & \bfseries 61.76 & 51.80 & 60.50 & 60.66 & {\bfseries\underline{71.33}} & 67.08 \\
			{} & {\scriptsize\textcolor{gray}{$\pm$2.68}} & {\scriptsize\textcolor{gray}{$\pm$2.26}} & {\scriptsize\textcolor{gray}{$\pm$2.06}} & {\scriptsize\textcolor{gray}{$\pm$1.61}} & {\scriptsize\textcolor{gray}{$\pm$4.83}} & {\scriptsize\textcolor{gray}{$\pm$0.70}} & {\scriptsize\textcolor{gray}{$\pm$0.70}} & {\scriptsize\textcolor{gray}{$\pm$1.92}} & {\scriptsize\textcolor{gray}{$\pm$2.60}} & {\scriptsize\textcolor{gray}{$\pm$3.61}} & {\scriptsize\textcolor{gray}{$\pm$3.05}} & {\scriptsize\textcolor{gray}{$\pm$0.62}} & {\scriptsize\textcolor{gray}{$\pm$1.07}} \\
			\midrule
			Avg Rank & 11.67 & 11.00 & 14.33 & 12.83 & \bfseries 4.00 & 6.00 & 4.67 & \bfseries 7.00 & 13.17 & 9.67 & 9.50 & {\bfseries\underline{1.67}} & 4.17 \\
			Avg AUAC & 58.07 & 59.68 & 51.03 & 56.36 & 67.81 & 67.88 & \bfseries 68.35 & \bfseries 66.75 & 54.91 & 63.04 & 62.98 & {\bfseries\underline{71.37}} & 69.02 \\
			\bottomrule
	\end{tabular}}
\end{table}
\begin{figure}[t]
    \centering
    \begin{subfigure}[b]{0.32\linewidth}
        \includegraphics[width=\linewidth]{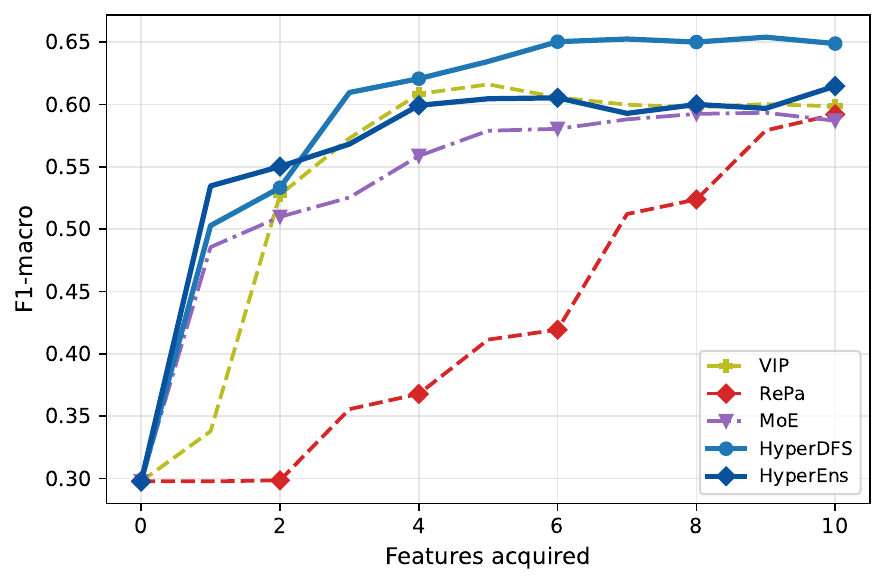}
        \caption{Diabetes}
    \end{subfigure}
    \hfill
    \begin{subfigure}[b]{0.32\linewidth}
        \includegraphics[width=\linewidth]{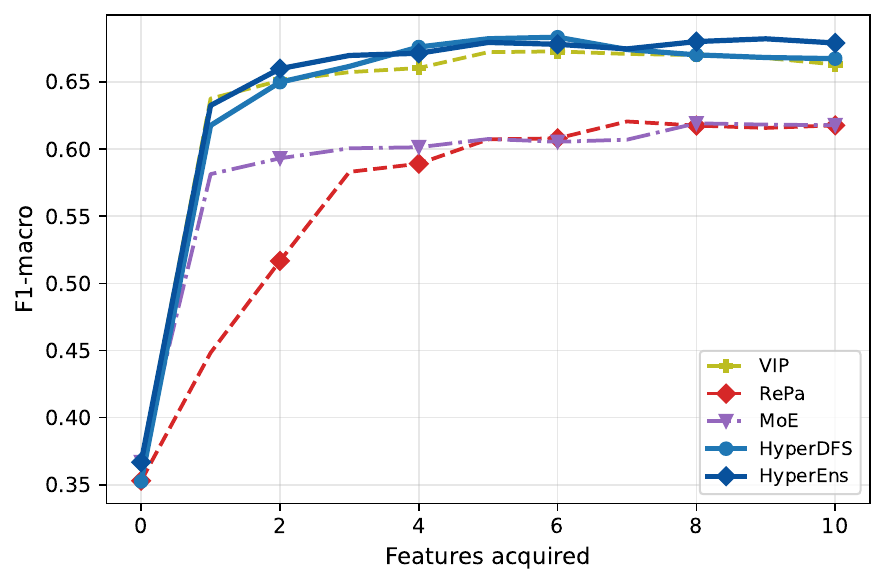}
        \caption{Metabric}
    \end{subfigure}
    \hfill
    \begin{subfigure}[b]{0.32\linewidth}
        \includegraphics[width=\linewidth]{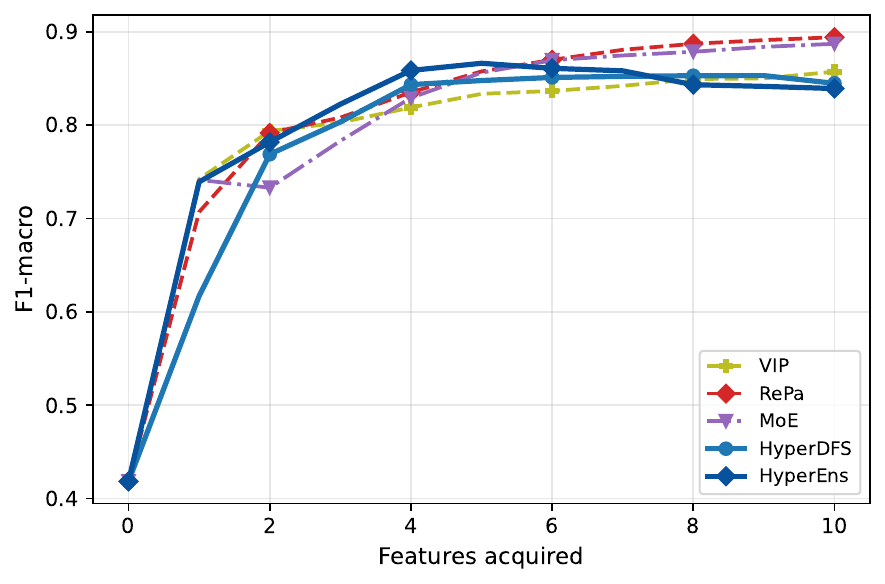}
        \caption{Miniboone}
    \end{subfigure}
    \caption{F1-macro according to the number of acquired features for the top-performing methods on three representative tabular datasets. Full acquisition curves for all datasets are provided in Appendix~\ref{apx:curves}.}
    \label{fig:acquisition_curves}
\end{figure}
\begin{table}[t]
	\centering
	\caption{AUAC-F1 on tabular benchmarks (mean $\pm$ std over 5-fold CV, \%).}
	\label{tab:tabular_auac_f1}
	\adjustbox{width=\linewidth}{
		\begin{tabular}{
				l
				*{7}{S[table-format=2.2, detect-weight=true]}
				|
				*{4}{S[table-format=2.2, detect-weight=true]}
				|
				*{2}{S[table-format=2.2, detect-weight=true]}
			}
			\toprule
			& \multicolumn{7}{c}{Single-model} & \multicolumn{4}{c}{Multi-model} & \multicolumn{2}{c}{Hypernetworks} \\
			\cmidrule(lr){2-8}\cmidrule(lr){9-12}\cmidrule(lr){13-14}
			{Dataset} & {CWCF} & {DIME} & {EDDI} & {INVASE} & {RePa} & {SEFA} & {VIP} & {AACO} & {Card.} & {Ensemble} & {MoE} & {HyperEns} & {Hyper-DFS} \\
			\midrule
			Bank & 54.31 & 63.20 & 63.20 & 70.12 & {\bfseries\underline{80.37}} & 79.73 & 78.39 & 79.19 & 50.18 & 80.15 & \bfseries 80.29 & \bfseries 80.26 & 79.73 \\
			{} & {\scriptsize\textcolor{gray}{$\pm$5.07}} & {\scriptsize\textcolor{gray}{$\pm$5.57}} & {\scriptsize\textcolor{gray}{$\pm$1.28}} & {\scriptsize\textcolor{gray}{$\pm$1.13}} & {\scriptsize\textcolor{gray}{$\pm$0.49}} & {\scriptsize\textcolor{gray}{$\pm$0.49}} & {\scriptsize\textcolor{gray}{$\pm$0.65}} & {\scriptsize\textcolor{gray}{$\pm$0.47}} & {\scriptsize\textcolor{gray}{$\pm$7.20}} & {\scriptsize\textcolor{gray}{$\pm$0.59}} & {\scriptsize\textcolor{gray}{$\pm$0.37}} & {\scriptsize\textcolor{gray}{$\pm$0.46}} & {\scriptsize\textcolor{gray}{$\pm$0.31}} \\
			California & 79.25 & 70.68 & 62.88 & 74.31 & 81.74 & \bfseries 82.25 & 79.86 & {\bfseries\underline{83.98}} & 69.37 & 83.60 & 82.46 & \bfseries 82.92 & 81.93 \\
			{} & {\scriptsize\textcolor{gray}{$\pm$1.39}} & {\scriptsize\textcolor{gray}{$\pm$4.31}} & {\scriptsize\textcolor{gray}{$\pm$1.50}} & {\scriptsize\textcolor{gray}{$\pm$1.45}} & {\scriptsize\textcolor{gray}{$\pm$1.02}} & {\scriptsize\textcolor{gray}{$\pm$0.25}} & {\scriptsize\textcolor{gray}{$\pm$0.61}} & {\scriptsize\textcolor{gray}{$\pm$0.29}} & {\scriptsize\textcolor{gray}{$\pm$6.98}} & {\scriptsize\textcolor{gray}{$\pm$0.34}} & {\scriptsize\textcolor{gray}{$\pm$0.50}} & {\scriptsize\textcolor{gray}{$\pm$0.46}} & {\scriptsize\textcolor{gray}{$\pm$0.76}} \\
			Cirrhosis & 44.72 & 46.41 & 39.11 & 43.25 & 49.92 & {\bfseries\underline{50.67}} & 49.64 & 47.10 & 45.63 & 49.50 & \bfseries 50.12 & \bfseries 49.64 & 48.66 \\
			{} & {\scriptsize\textcolor{gray}{$\pm$3.04}} & {\scriptsize\textcolor{gray}{$\pm$4.66}} & {\scriptsize\textcolor{gray}{$\pm$3.01}} & {\scriptsize\textcolor{gray}{$\pm$2.50}} & {\scriptsize\textcolor{gray}{$\pm$2.33}} & {\scriptsize\textcolor{gray}{$\pm$2.89}} & {\scriptsize\textcolor{gray}{$\pm$2.10}} & {\scriptsize\textcolor{gray}{$\pm$3.43}} & {\scriptsize\textcolor{gray}{$\pm$1.60}} & {\scriptsize\textcolor{gray}{$\pm$2.22}} & {\scriptsize\textcolor{gray}{$\pm$2.46}} & {\scriptsize\textcolor{gray}{$\pm$3.13}} & {\scriptsize\textcolor{gray}{$\pm$3.27}} \\
			Diabetes & 37.20 & 34.21 & 29.75 & 30.57 & 45.10 & 56.18 & \bfseries 59.18 & 56.13 & 50.72 & 53.52 & \bfseries 56.81 & 59.24 & {\bfseries\underline{62.81}} \\
			{} & {\scriptsize\textcolor{gray}{$\pm$6.54}} & {\scriptsize\textcolor{gray}{$\pm$3.11}} & {\scriptsize\textcolor{gray}{$\pm$0.01}} & {\scriptsize\textcolor{gray}{$\pm$1.42}} & {\scriptsize\textcolor{gray}{$\pm$7.85}} & {\scriptsize\textcolor{gray}{$\pm$2.54}} & {\scriptsize\textcolor{gray}{$\pm$1.80}} & {\scriptsize\textcolor{gray}{$\pm$2.05}} & {\scriptsize\textcolor{gray}{$\pm$8.44}} & {\scriptsize\textcolor{gray}{$\pm$4.33}} & {\scriptsize\textcolor{gray}{$\pm$2.90}} & {\scriptsize\textcolor{gray}{$\pm$4.65}} & {\scriptsize\textcolor{gray}{$\pm$1.99}} \\
			Heart & 72.19 & 73.45 & 68.00 & 71.98 & 81.55 & 76.12 & \bfseries 82.13 & 79.26 & 77.71 & 82.05 & \bfseries 82.81 & {\bfseries\underline{82.89}} & 80.80 \\
			{} & {\scriptsize\textcolor{gray}{$\pm$7.97}} & {\scriptsize\textcolor{gray}{$\pm$6.10}} & {\scriptsize\textcolor{gray}{$\pm$5.99}} & {\scriptsize\textcolor{gray}{$\pm$5.03}} & {\scriptsize\textcolor{gray}{$\pm$3.93}} & {\scriptsize\textcolor{gray}{$\pm$4.17}} & {\scriptsize\textcolor{gray}{$\pm$4.45}} & {\scriptsize\textcolor{gray}{$\pm$5.41}} & {\scriptsize\textcolor{gray}{$\pm$4.54}} & {\scriptsize\textcolor{gray}{$\pm$4.77}} & {\scriptsize\textcolor{gray}{$\pm$4.35}} & {\scriptsize\textcolor{gray}{$\pm$4.63}} & {\scriptsize\textcolor{gray}{$\pm$6.07}} \\
			Metabric & 36.68 & 46.82 & 45.11 & 45.22 & 59.73 & 63.20 & \bfseries 66.51 & 61.19 & 52.58 & \bfseries 62.74 & 60.78 & {\bfseries\underline{67.49}} & 67.03 \\
			{} & {\scriptsize\textcolor{gray}{$\pm$0.04}} & {\scriptsize\textcolor{gray}{$\pm$6.47}} & {\scriptsize\textcolor{gray}{$\pm$6.23}} & {\scriptsize\textcolor{gray}{$\pm$2.52}} & {\scriptsize\textcolor{gray}{$\pm$4.98}} & {\scriptsize\textcolor{gray}{$\pm$1.09}} & {\scriptsize\textcolor{gray}{$\pm$2.03}} & {\scriptsize\textcolor{gray}{$\pm$1.60}} & {\scriptsize\textcolor{gray}{$\pm$11.39}} & {\scriptsize\textcolor{gray}{$\pm$1.58}} & {\scriptsize\textcolor{gray}{$\pm$2.89}} & {\scriptsize\textcolor{gray}{$\pm$1.70}} & {\scriptsize\textcolor{gray}{$\pm$0.81}} \\
			Miniboone & 48.69 & 50.91 & 51.26 & 46.95 & \bfseries 85.72 & 71.57 & 83.15 & 85.05 & 82.09 & {\bfseries\underline{86.25}} & 84.38 & \bfseries 84.13 & 83.52 \\
			{} & {\scriptsize\textcolor{gray}{$\pm$6.29}} & {\scriptsize\textcolor{gray}{$\pm$6.26}} & {\scriptsize\textcolor{gray}{$\pm$6.03}} & {\scriptsize\textcolor{gray}{$\pm$3.07}} & {\scriptsize\textcolor{gray}{$\pm$0.62}} & {\scriptsize\textcolor{gray}{$\pm$4.60}} & {\scriptsize\textcolor{gray}{$\pm$0.12}} & {\scriptsize\textcolor{gray}{$\pm$0.22}} & {\scriptsize\textcolor{gray}{$\pm$0.94}} & {\scriptsize\textcolor{gray}{$\pm$0.41}} & {\scriptsize\textcolor{gray}{$\pm$0.30}} & {\scriptsize\textcolor{gray}{$\pm$1.10}} & {\scriptsize\textcolor{gray}{$\pm$0.53}} \\
			Wine & 86.29 & 86.14 & 75.98 & 80.88 & 94.95 & \bfseries 95.81 & 92.32 & \bfseries 95.45 & 93.29 & 94.67 & 95.28 & {\bfseries\underline{96.00}} & 93.31 \\
			{} & {\scriptsize\textcolor{gray}{$\pm$5.76}} & {\scriptsize\textcolor{gray}{$\pm$4.37}} & {\scriptsize\textcolor{gray}{$\pm$3.11}} & {\scriptsize\textcolor{gray}{$\pm$9.11}} & {\scriptsize\textcolor{gray}{$\pm$2.77}} & {\scriptsize\textcolor{gray}{$\pm$3.35}} & {\scriptsize\textcolor{gray}{$\pm$2.28}} & {\scriptsize\textcolor{gray}{$\pm$2.21}} & {\scriptsize\textcolor{gray}{$\pm$1.73}} & {\scriptsize\textcolor{gray}{$\pm$1.14}} & {\scriptsize\textcolor{gray}{$\pm$2.20}} & {\scriptsize\textcolor{gray}{$\pm$2.47}} & {\scriptsize\textcolor{gray}{$\pm$4.20}} \\
			Yeast & 34.86 & 35.17 & 20.42 & 26.99 & \bfseries 45.35 & 42.84 & 39.15 & 40.31 & 34.08 & 43.64 & {\bfseries\underline{45.86}} & 41.95 & \bfseries 42.20 \\
			{} & {\scriptsize\textcolor{gray}{$\pm$5.69}} & {\scriptsize\textcolor{gray}{$\pm$3.96}} & {\scriptsize\textcolor{gray}{$\pm$4.62}} & {\scriptsize\textcolor{gray}{$\pm$3.18}} & {\scriptsize\textcolor{gray}{$\pm$3.92}} & {\scriptsize\textcolor{gray}{$\pm$3.24}} & {\scriptsize\textcolor{gray}{$\pm$1.86}} & {\scriptsize\textcolor{gray}{$\pm$3.57}} & {\scriptsize\textcolor{gray}{$\pm$3.72}} & {\scriptsize\textcolor{gray}{$\pm$1.87}} & {\scriptsize\textcolor{gray}{$\pm$1.46}} & {\scriptsize\textcolor{gray}{$\pm$1.91}} & {\scriptsize\textcolor{gray}{$\pm$2.53}} \\
			\midrule
			Avg Rank & 17.33 & 16.78 & 19.00 & 18.00 & \bfseries 6.89 & 7.56 & 9.78 & 9.33 & 15.33 & 6.00 & \bfseries 4.89 & {\bfseries\underline{4.33}} & 7.11 \\
			Avg AUAC & 54.91 & 56.33 & 50.63 & 54.47 & 69.38 & 68.71 & \bfseries 70.03 & 69.74 & 61.74 & 70.68 & \bfseries 70.98 & {\bfseries\underline{71.61}} & 71.11 \\
			\bottomrule
	\end{tabular}}
\end{table}
\paragraph{Tabular datasets.} Results on synthetic benchmarks appear in Table~\ref{tab:synthetic_auac_f1} and on real-world tabular datasets in Table~\ref{tab:tabular_auac_f1}. HyperEns achieves the best average rank and the highest mean AUAC-F1 on both benchmarking suites, confirming that per-subset specialisation via hypernetworks combined with ensembling results in consistent gains across diverse task scenarios. Moreover, even without ensembling \textsc{Hyper-DFS} performs on average better than the strongest single-model baselines (RePa, VIP) and reaches the second-best average AUAC-F1. Among single-model methods, RePa and VIP are positioned as the strongest competitors across both suites, with VIP achieving the highest average AUAC-F1 and RePa the best ranking within that group. In the multi-model approaches, Ensemble and MoE perform comparably to VIP, and substantially outperform CardinalDFS, which confirms that static cardinality-based routing is not a good strategy overall.

\begin{wraptable}{r}{0.5\linewidth} 
    \centering
    \caption{AUAC-F1 on held-out feature subsets.}
    \label{tab:zeroshot_auac_f1_small}
    \small 
    \begin{tabular}{lccc}
        \toprule
        Dataset & VIP & MoE & Hyper-DFS \\
        \midrule
        Avg Rank (Tab) & 2.67 & 1.78 & \textbf{1.56} \\
        Avg AUAC (Tab) & 46.55& 63.11 & \textbf{64.06} \\
        \midrule 
        Avg Rank (Img) & 2.00 & 3.00 & \textbf{1.00} \\
        Avg AUAC (Img) & 85.02 & 83.94 & \textbf{87.57} \\
        \bottomrule
    \end{tabular}
\end{wraptable}
Figure~\ref{fig:acquisition_curves} presents acquisition curves for three datasets chosen to illustrate the range of observed behaviours. On Diabetes, \textsc{Hyper-DFS} achieves the largest gains, outperforming all other methods. On Metabric, both \textsc{Hyper-DFS} versions maintain a persistent advantage over all competitors, with the gap being largest at low budgets where per-subset specialisation is most valuable. On Miniboone, the hypernetwork methods dominate at low budgets but are overtaken by RePa and MoE beyond eight features. In this case, subset specialisation through the hypernetwork performs better with fewer features, suggesting that the weight sharing mechanism of hypernetworks can be more effective at lower cardinalities. In Appendix~\ref{apx:additional_results}, Table \ref{tab:delta_auac_f1}, we confirm this.

\begin{table}[t]
	\centering
	\caption{AUAC-F1 on image patch benchmarks (mean $\pm$ std over CV folds, \%).}
	\label{tab:image_auac_auac_f1}
	\adjustbox{width=\linewidth}{
		\begin{tabular}{
				l
				*{4}{S[table-format=2.2, detect-weight=true]}
				|
				*{4}{S[table-format=2.2, detect-weight=true]}
				|
				*{2}{S[table-format=2.2, detect-weight=true]}
			}
			\toprule
			& \multicolumn{4}{c}{Single-model} & \multicolumn{4}{c}{Multi-model} & \multicolumn{2}{c}{Hypernetworks} \\
			\cmidrule(lr){2-5}\cmidrule(lr){6-9}\cmidrule(lr){10-11}
			{Dataset} & {DIME} & {RePa} & {SEFA} & {VIP} & {AACO} & {Card.} & {Ensemble} & {MoE} & {HyperEns} & {Hyper-DFS} \\
			\midrule
			FashionMNIST & 73.96 & 72.60 & 69.09 & \bfseries 80.80 & \bfseries 80.36 & 72.81 & 73.09 & 72.57 & 80.25 & {\bfseries\underline{81.19}} \\
			{} & {\scriptsize\textcolor{gray}{$\pm$1.35}} & {\scriptsize\textcolor{gray}{$\pm$1.74}} & {\scriptsize\textcolor{gray}{$\pm$2.52}} & {\scriptsize\textcolor{gray}{$\pm$0.50}} & {\scriptsize\textcolor{gray}{$\pm$0.50}} & {\scriptsize\textcolor{gray}{$\pm$1.77}} & {\scriptsize\textcolor{gray}{$\pm$1.75}} & {\scriptsize\textcolor{gray}{$\pm$0.53}} & {\scriptsize\textcolor{gray}{$\pm$0.58}} & {\scriptsize\textcolor{gray}{$\pm$0.25}} \\
			Imagenette & 94.05 & 93.90 & 94.14 & \bfseries 94.14 & 94.42 & 94.02 & 94.57 & \bfseries 94.57 & {\bfseries\underline{94.98}} & 94.89 \\
			{} & {\scriptsize\textcolor{gray}{$\pm$0.61}} & {\scriptsize\textcolor{gray}{$\pm$0.30}} & {\scriptsize\textcolor{gray}{$\pm$0.47}} & {\scriptsize\textcolor{gray}{$\pm$0.40}} & {\scriptsize\textcolor{gray}{$\pm$0.38}} & {\scriptsize\textcolor{gray}{$\pm$0.63}} & {\scriptsize\textcolor{gray}{$\pm$0.47}} & {\scriptsize\textcolor{gray}{$\pm$0.35}} & {\scriptsize\textcolor{gray}{$\pm$0.37}} & {\scriptsize\textcolor{gray}{$\pm$0.32}} \\
			MNIST & 84.37 & 83.77 & 77.10 & \bfseries 90.41 & \bfseries 90.04 & 83.66 & 84.52 & 81.56 & 90.44 & {\bfseries\underline{90.54}} \\
			{} & {\scriptsize\textcolor{gray}{$\pm$1.34}} & {\scriptsize\textcolor{gray}{$\pm$1.30}} & {\scriptsize\textcolor{gray}{$\pm$2.48}} & {\scriptsize\textcolor{gray}{$\pm$0.28}} & {\scriptsize\textcolor{gray}{$\pm$0.23}} & {\scriptsize\textcolor{gray}{$\pm$1.62}} & {\scriptsize\textcolor{gray}{$\pm$2.24}} & {\scriptsize\textcolor{gray}{$\pm$2.92}} & {\scriptsize\textcolor{gray}{$\pm$0.23}} & {\scriptsize\textcolor{gray}{$\pm$0.64}} \\
			PCam & 70.08 & 69.91 & \bfseries 71.13 & 68.58 & 68.98 & \bfseries 72.35 & 71.30 & 70.98 & 72.88 & {\bfseries\underline{73.39}} \\
			{} & {\scriptsize\textcolor{gray}{$\pm$1.12}} & {\scriptsize\textcolor{gray}{$\pm$1.77}} & {\scriptsize\textcolor{gray}{$\pm$0.77}} & {\scriptsize\textcolor{gray}{$\pm$3.03}} & {\scriptsize\textcolor{gray}{$\pm$0.50}} & {\scriptsize\textcolor{gray}{$\pm$0.43}} & {\scriptsize\textcolor{gray}{$\pm$0.59}} & {\scriptsize\textcolor{gray}{$\pm$0.75}} & {\scriptsize\textcolor{gray}{$\pm$0.25}} & {\scriptsize\textcolor{gray}{$\pm$0.71}} \\
			SVHN & 64.58 & 63.57 & 55.39 & {\bfseries\underline{69.99}} & \bfseries 64.31 & 63.81 & 63.36 & 62.13 & 69.05 & \bfseries 69.27 \\
			{} & {\scriptsize\textcolor{gray}{$\pm$3.19}} & {\scriptsize\textcolor{gray}{$\pm$0.40}} & {\scriptsize\textcolor{gray}{$\pm$1.30}} & {\scriptsize\textcolor{gray}{$\pm$0.98}} & {\scriptsize\textcolor{gray}{$\pm$0.13}} & {\scriptsize\textcolor{gray}{$\pm$1.88}} & {\scriptsize\textcolor{gray}{$\pm$3.45}} & {\scriptsize\textcolor{gray}{$\pm$1.94}} & {\scriptsize\textcolor{gray}{$\pm$0.37}} & {\scriptsize\textcolor{gray}{$\pm$1.49}} \\
			\bottomrule
		\end{tabular}
	}
\end{table}

\paragraph{Image datasets.} Table~\ref{tab:image_auac_auac_f1} reports AUAC-F1 on image patch benchmarks. For this comparison, we drop CWCF, EDDI and INVASE because of computational reasons. The \textsc{Hyper-DFS} variants achieved the highest score on four of five datasets, and the gap on SVHN to the best method (VIP) is of only 0.72 points. On PCam, \textsc{Hyper-DFS} outperforms all baselines, including dedicated multi-model approaches by over one point. Notably, conventional multi-model methods do not consistently improve over strong single-model baselines on these datasets, and HyperEns likewise does not consistently outperform single-model \textsc{Hyper-DFS}. Among multi-model baselines, AACO is competitive on FashionMNIST and MNIST but does not generalise this advantage to the remaining datasets. Overall, these results indicate that the benefits of per-subset specialisation extend beyond tabular data to spatially structured inputs like images, and highlight the limitations of standard multi-model approaches.

\begin{figure}[t]
    \centering
    \begin{subfigure}[b]{0.32\linewidth}
        \includegraphics[width=\linewidth]{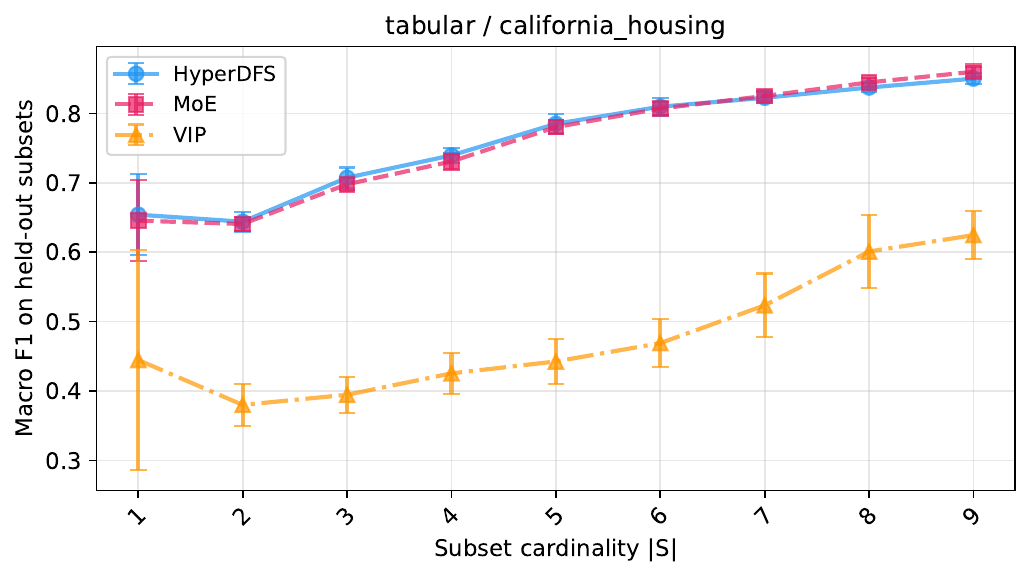}
    \end{subfigure}
    \hfill
    \begin{subfigure}[b]{0.32\linewidth}
        \includegraphics[width=\linewidth]{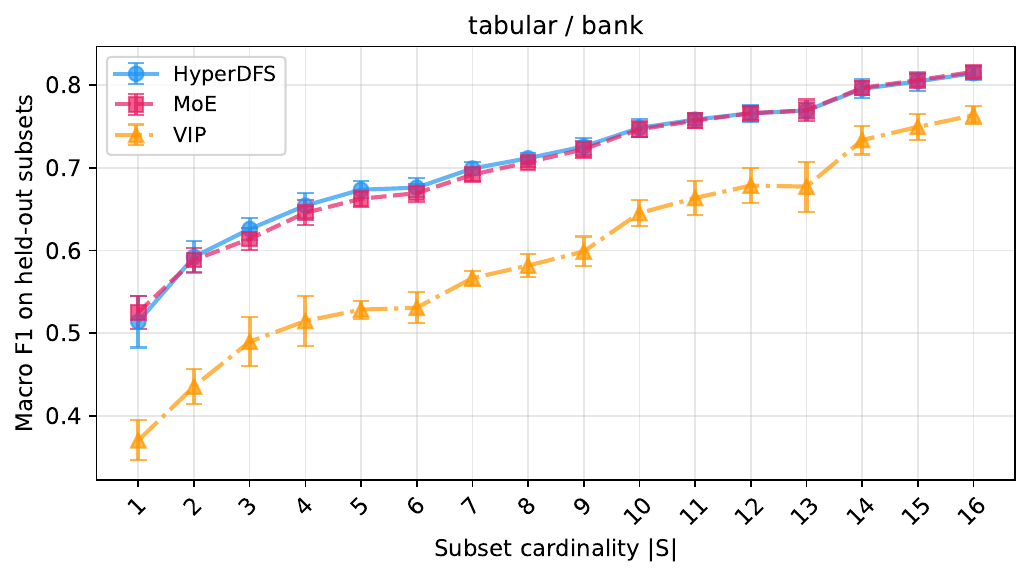}
    \end{subfigure}
    \hfill
    \begin{subfigure}[b]{0.32\linewidth}
        \includegraphics[width=\linewidth]{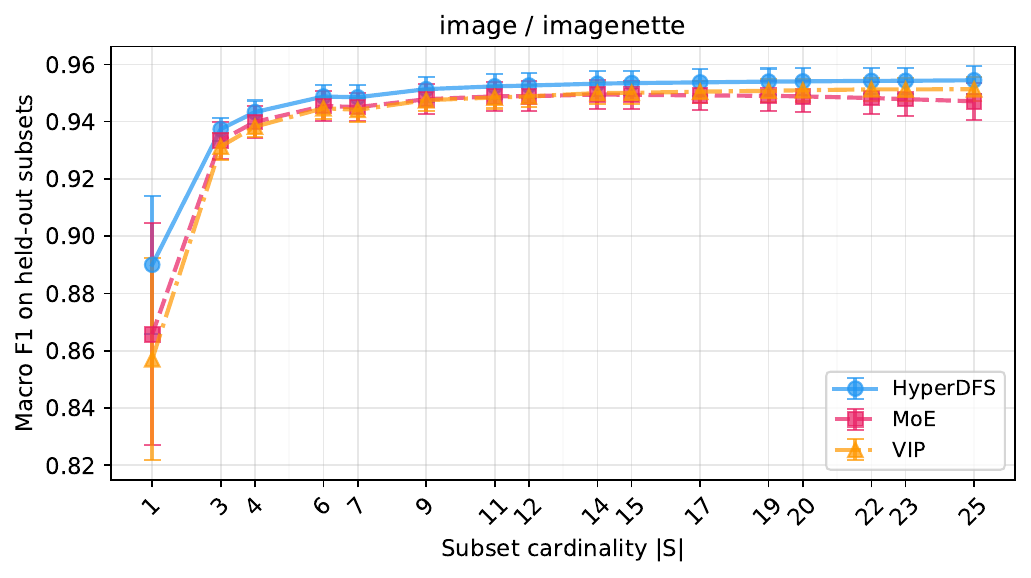}
    \end{subfigure}
    \caption{F1 score for unseen feature subsets in training, ordered according to the set cardinality. The three methods are shown: \textsc{Hyper-DFS}, MoE and VIP.}
    \label{fig:acquisition_curves_zeroshot}
\end{figure}

\paragraph{Zero-shot subset generalisation.}
Table~\ref{tab:zeroshot_auac_f1_small} reports zero-shot AUAC-F1 on held-out feature subsets across tabular and image benchmarks, comparing the strongest single-model baseline (VIP) and multi-model baseline (MoE) against \textsc{Hyper-DFS}. The results reveal a gap between single-model and multi-model approaches: VIP degrades substantially in the zero-shot setting on tabular data, while both MoE and \textsc{Hyper-DFS} surpass it by more than 15 points on average. Figures~\ref{fig:acquisition_curves_zeroshot}(a) and~\ref{fig:acquisition_curves_zeroshot}(b) illustrate two examples in tabular data in which VIP lags behind across all subset cardinalities, whereas MoE and \textsc{Hyper-DFS} remain competitive. On image datasets the relative ordering shifts: VIP recovers and outperforms MoE, yet \textsc{Hyper-DFS} consistently achieves the best performance. Figure~\ref{fig:acquisition_curves_zeroshot}(c) illustrates this pattern on Imagenette, where \textsc{Hyper-DFS} maintains a consistent advantage over both baselines across all subset cardinalities.

\begin{table}[t]
	\centering
	\caption{Performance for tabular data using different encodings of the knowledge status for the hypernetwork input.}
	\label{tab:abl_encoding}
	\adjustbox{width=\linewidth}{%
		\begin{tabular}{lccccccccccc}
			\toprule
			Variant & Bank & California & Cirrhosis & Diabetes & Heart & Metabric & Miniboone & Wine & Yeast & Overall & Rank \\
			\midrule
			DeepSets & 80.29 & 83.03 & \textbf{49.68} & 57.68 & 80.95 & 65.79 & 84.22 & 93.21 & 41.38 & 70.69 & 2.44 \\
			{} & {\scriptsize\textcolor{gray}{$\pm$0.80}} & {\scriptsize\textcolor{gray}{$\pm$0.86}} & {\scriptsize\textcolor{gray}{$\pm$3.27}} & {\scriptsize\textcolor{gray}{$\pm$2.56}} & {\scriptsize\textcolor{gray}{$\pm$5.40}} & {\scriptsize\textcolor{gray}{$\pm$2.25}} & {\scriptsize\textcolor{gray}{$\pm$1.50}} & {\scriptsize\textcolor{gray}{$\pm$2.86}} & {\scriptsize\textcolor{gray}{$\pm$2.98}} & {\scriptsize\textcolor{gray}{$\pm$17.02}} & {} \\
			RawMask & 79.98 & \textbf{83.15} & 49.05 & 58.44 & 80.80 & 66.40 & 85.11 & 94.54 & 41.89 & 71.04 & 2.22 \\
			{} & {\scriptsize\textcolor{gray}{$\pm$0.52}} & {\scriptsize\textcolor{gray}{$\pm$1.13}} & {\scriptsize\textcolor{gray}{$\pm$2.85}} & {\scriptsize\textcolor{gray}{$\pm$2.93}} & {\scriptsize\textcolor{gray}{$\pm$3.80}} & {\scriptsize\textcolor{gray}{$\pm$1.46}} & {\scriptsize\textcolor{gray}{$\pm$0.70}} & {\scriptsize\textcolor{gray}{$\pm$3.61}} & {\scriptsize\textcolor{gray}{$\pm$2.63}} & {\scriptsize\textcolor{gray}{$\pm$17.13}} & {} \\
			Proposed & \textbf{80.31} & 82.65 & 49.57 & \textbf{60.46} & \textbf{81.39} & \textbf{66.99} & \textbf{85.74} & \textbf{94.76} & \textbf{42.20} & \textbf{71.56} & \textbf{1.33} \\
			{} & {\scriptsize\textcolor{gray}{$\pm$0.45}} & {\scriptsize\textcolor{gray}{$\pm$1.13}} & {\scriptsize\textcolor{gray}{$\pm$3.03}} & {\scriptsize\textcolor{gray}{$\pm$1.42}} & {\scriptsize\textcolor{gray}{$\pm$3.68}} & {\scriptsize\textcolor{gray}{$\pm$1.24}} & {\scriptsize\textcolor{gray}{$\pm$0.32}} & {\scriptsize\textcolor{gray}{$\pm$3.01}} & {\scriptsize\textcolor{gray}{$\pm$4.34}} & {\scriptsize\textcolor{gray}{$\pm$16.93}} & {} \\
			\bottomrule
		\end{tabular}
	}%
\end{table}

\begin{table}
	\centering
	\caption{Spearman correlations between distances in the conditioning space, predictor weights, and functional outputs.}
	\label{tab:conditioning_smoothness}
	\begin{tabular}{lcccc}
		\toprule
		Dataset ($M$) &
		$\rho(d_z, d_\theta)$ &
		$\rho(d_z, d_f)$ &
		$\rho(d_H, d_\theta)$ &
		$\rho(d_H, d_f)$ \\
		\midrule
		Heart (13)            & 0.975 & 0.922 & 0.694 & 0.648 \\
		Synthetic pairs (12)  & 0.982 & 0.976 & 0.642 & 0.634 \\
		Diabetes (45)         & 0.983 & 0.984 & 0.539 & 0.537 \\
		MNIST (49)            & 0.996 & 0.997 & 0.572 & 0.552 \\
		\bottomrule
	\end{tabular}
\end{table}
\begin{figure}[h]
	\centering
	\begin{subfigure}[t]{0.48\linewidth}
		\centering
		\includegraphics[width=\linewidth]{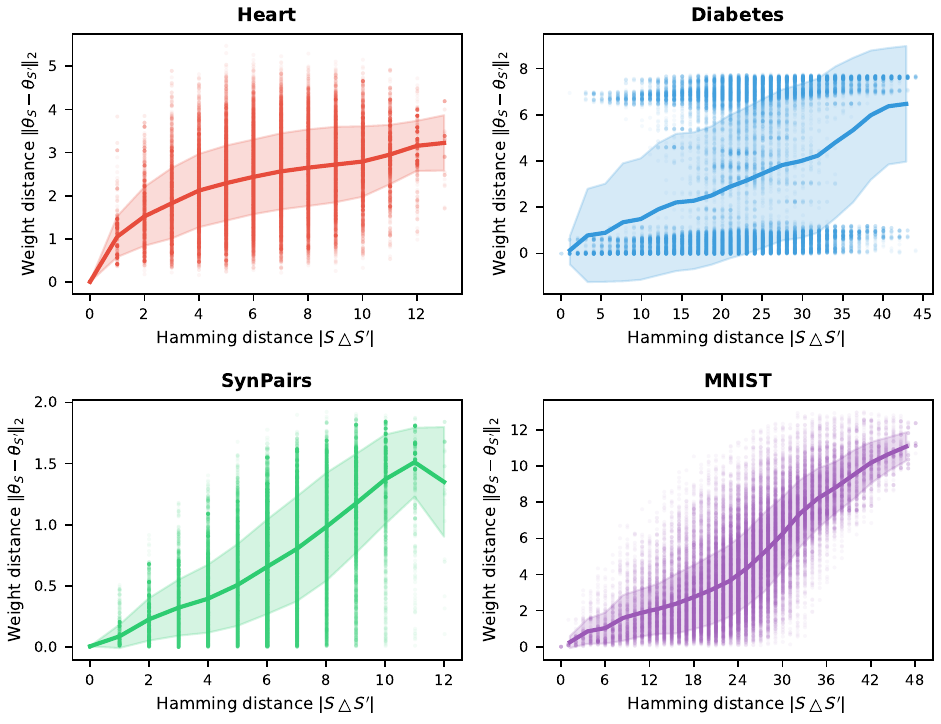}
		\caption{Hamming distance in the original space versus primary network weight distance.}
		\label{fig:weight_vs_hamming_a}
	\end{subfigure}
	\hspace{0.01\linewidth}
	\vrule width 0.3pt
	\hspace{0.01\linewidth}
	\begin{subfigure}[t]{0.48\linewidth}
		\centering
		\includegraphics[width=\linewidth]{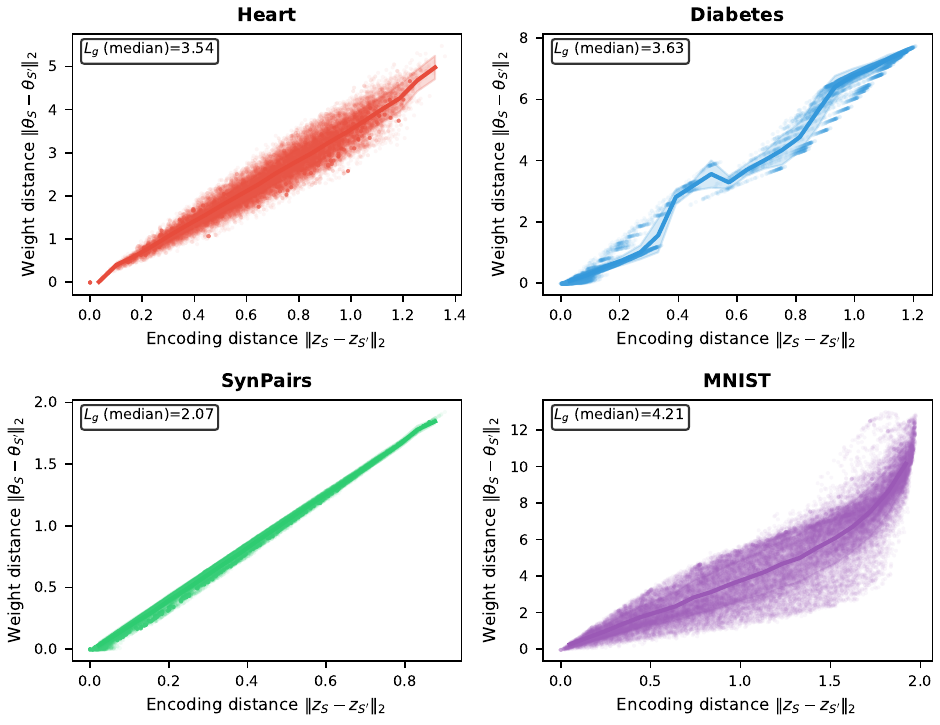}
		\caption{Hamming distance in the Set Transformer-based encoding space versus primary network weight distance.}
		\label{fig:weight_vs_hamming_b}
	\end{subfigure}
	
	\caption{Effects of Set Transformers encoding in the weight representation space.}
	\label{fig:weight_vs_hamming}
\end{figure}
\paragraph{Conditioning representation.} We compare the effects of using different task embedding vectors. For every pair of feature subsets $(S,S')$, we measure three distances:
the cosine distance between their conditioning vectors,
$d_z(S,S') = 1 - \cos(z_S, z_{S'})$; between the corresponding generated primary-network weights,
$d_\theta(S,S') = 1 - \cos(\theta_S, \theta_{S'})$; and the raw Hamming distance $d_H(S,S')$. We additionally measure a functional distance $d_f(S,S')$, defined as the mean Jensen--Shannon divergence between the predictive distributions of the two generated networks over a common set of reference inputs. Table \ref{tab:conditioning_smoothness} shows that indeed ISAB embeddings show a stronger correlation between embedding and task distances using both Hamming and cosine distance. Figure \ref{fig:weight_vs_hamming} also shows the distances between models align much better with distances in the task encoding space when using the Set Transformer's encoder than the raw masks. 

\section{Related Work}
\label{sec:related_work}

The idea of using multiple classifiers in a DFS setting has been partially explored in AACO~\citep{valancius2024acquisition}, which constructs a dictionary of precomputed subset-specific classifiers for lower subset cardinalities. However, this method still uses hundreds of classifiers plus the shared amortised predictor for a higher number of features. An approach to subset adaptation was proposed in~\citep{fumanal2025dynamic}, where a small reparametrisation of the predictor's final layers adjusts the model to each observed subset, following a similar principle to the hypernetwork parametrisation of~\citep{gonzalezortiz2024mip}. However, although it is computationally efficient, the expressiveness of such reparametrisation is inherently constrained: the adapted model remains anchored to the base parameters and cannot represent classifiers that require substantially different decision boundaries for different subsets.

Pattern-by-pattern approaches, where a model is trained for each missingness pattern, i.e., optimising Eq.~\eqref{eq:ideal_obj}, show significantly better results than amortised solutions in some cases \citep{Sell2024}. In \citet{lobo2025primer}, the authors also show that optimising Eq.~\eqref{eq:single_model_obj} might never reach the optimal solution for logistic regression models. Less computationally intensive pattern-by-pattern approaches, such as threshold non-relevant feature interactions \citep{Sell2024} or parameter-sharing schemes \citep{stempfle23}, have shown improved results in reducing the bias of amortised solutions while controlling the variance associated with pattern-by-pattern methods. However, these approaches are still limited in complex non-linear problems with a potentially exponential number of marginal distributions. Compared to our proposal, the hypernetwork captures shared structure across subsets, while still allowing enough flexibility to adapt to individual missingness patterns, providing a middle ground between amortised and expensive pattern-by-pattern approaches.

In the broader hypernetwork literature, there have been proposals to use the conditioning vector as a task~\citep{zhao2020meta} as a class specification \citep{przewikezlikowski2024hypermaml, chauhan2024hypernetworks} and even a support set of a dataset can be enough to generate weights for a model \citep{bonet2024hyperfast}. Hypernetworks have also been used for context-conditional feature selection \citep{sristicontextual}. These approaches have the advantage that the cost of generating many solutions is amortised through the hypernetwork \citep{beck2023recurrent}. This amortisation property is particularly relevant for DFS: even with unlimited memory to store all $2^M$ possible models, the vast majority of subsets would receive insufficient gradient signal during training to learn effective parameters. However, prior work in multi-task learning has also shown that multiple tasks might have cancelling gradients or training instabilities \citep{shigradient}, which existing literature on multi-task hypernetworks has addressed through pretrained components or by restricting the hypernetwork to low-rank adaptation~\citep{zhao2020meta, przewikezlikowski2024hypermaml}. In \textsc{Hyper-DFS}, we mitigate gradient interference instead through a controlled mask distribution per mini-batch combined with a learning-rate warmup schedule, which avoids both component-wise training stages and capacity restrictions on the hypernetwork.

\section{Conclusions and Limitations}
In this paper we argued that per-subset specialisation is a structural limitation of existing dynamic feature selection, as no shared predictor can simultaneously be optimal across an exponential family of feature subsets. We also showed that existing mask-concatenation approaches are instances of embedding-based architectures and inherit the modularity penalty of \citet{galanti2020modularity}. We then proposed \textsc{Hyper-DFS}, a framework in which a Set Transformer encodes feature subsets into a continuous representation and a hypernetwork generates a dedicated primary network on demand. The continuous encoding produces a smooth conditioning space in which functionally similar observation masks result in functionally similar primary networks, something that the pure binary-mask formulation cannot do. We empirically support our results across synthetic, tabular, and image benchmarks, where \textsc{Hyper-DFS} competitive or better results than state-of-the-art DFS baselines, and multi-model baselines like ensembles and MoE.

Our contribution has limitations as well. The same smoothness that enables generalisation also limits diversity: hyperensembling over perturbed subset encodings produced no measurable gain. Training a hypernetwork is more expensive than training a single primary network of comparable capacity: across the tabular datasets, the complete Hyper-DFS model contains 0.68M–6.70M parameters, although the generated primary predictor itself contains only 4.9k–11.3k. A tighter theoretical account of how the per-subset gap distributes across subsets would help understand the performance ceiling of each method. The scope of our experimentation did not allow us to find statistical differences between methods, which a larger number of datasets would have provided. 
\section{Acknowledgement}

This research and Javier Fumanal-Idocin were supported by EU Horizon Europe under the Marie Skłodowska-Curie COFUND grant No 101081327 YUFE4Postdocs. Raquel Fernandez-Peralta is funded by the EU NextGenerationEU through the Recovery and Resilience Plan for Slovakia under the project No. 09I03-03-V04- 00557.

The authors acknowledge the use of the High Performance Computing Facility (Ceres) and its associated support services at the University of Essex in the completion of this work.

\bibliography{neurips_2024}

\clearpage
\appendix


\section*{Appendix Overview}

	The appendix is organised into two parts:
    
	\paragraph{Further Analysis of the primary networks weights in \textsc{Hyper-DFS}.} These appendices cover the variance and gradient analysis of the primary network. 
    \begin{itemize}
        \item Appendix~\ref{apx:weight_variance} analyses the weight variance in the hypernetwork, and why using both absence and presence embeddings mitigates variance issues.
        \item Appendix \ref{app:hyperensembling} discusses loss enhancement to prevent representational collapse in the hypernetwork. 
        \item Appendix \ref{apx:grad_dilution} studies the problem of gradient dilution and its solutions.
    \end{itemize}
	 
	\textbf{Additional Empirical Results and Reproducibility.} Includes additional experimental details.
    \begin{itemize}
        \item Appendix~\ref{apx:additional_results} provides the full results for the experiments discussed in the main body of the paper.
        \item Appendix~\ref{apx:implementation} provides full implementation details, including dataset preprocessing, model architectures, hyperparameter ranges, and training configurations.
    \end{itemize}
	
	\bigskip

\section{Controlling Weight Variance in \textsc{Hyper-DFS}} \label{apx:weight_variance}

In this appendix, we study how to control the weight variance through the hypernetwork training. For that, we propose two things: to control the norm of the hypernetwork input, and to use a scale regularisation loss. Both decisions are described in detail in the following. 

\subsection{The Subset-Size Bias Under Sum Aggregation in Set Transformer Encoding}
Training hypernetworks is known to be harder than training standard networks of comparable size~\citep{chang2020principled}, because the gradient signal reaching $\phi$ is mediated by the primary network:
\begin{equation}
    \nabla_\phi \mathcal{L}
    \;=\;
    J_\phi^\top\, \nabla_\theta \mathcal{L},
    \qquad
    J_\phi
    \;=\;
    \frac{\partial g_\phi(\tilde{\mathbf{z}}_{\
    omega,\rho}(\mathbf{I}_S))}{\partial \phi}
    \;\in\; \mathbb{R}^{|\theta| \times |\phi|},
\end{equation}
where $J_\phi$ is the Jacobian of the weight-generation function with respect to the hypernetwork parameters. If the element-wise variance of $\theta_S = g_\phi(\tilde{\mathbf{z}}_{\omega,\rho}(\mathbf{I}_S))$ deviates substantially from the Xavier target ($\mathrm{Var}(W_{ij}) = 1/d_{\text{in}}$, where $d_{\text{in}}$ is the input dimension of the corresponding layer), the primary network's pre-activations either vanish or saturate, suppressing $\nabla_\theta \mathcal{L}$, and the gradient reaching $\phi$~\citep{glorot2010understanding}. Existing work~\citep{chang2020principled} addresses this by deriving conditions on $g_\phi$'s weight matrices so that the marginal distribution of the generated weights matches this target. However, their analysis treats the conditioning input distribution as fixed. This is problematic in a DFS setting, where the conditioning signal has systematically different norms depending on the cardinality of the knowledge status. Two complementary mechanisms can mitigate this problem: one is to normalise $\mathbf{z}_{\omega, \rho}(\mathbf{I}_S)$ (Eq.~\eqref{eq:subset_encoding}), the other is to decouple its norm from $|S|$, which we achieve using absence and presence embeddings for each feature. In the following, we formalise how this norm and cardinality relationship affects the variance of the weights of the primary network and how our solutions are mitigating this problem.

\begin{proposition}[Cardinality-dependent weight scale under different encoder designs]
\label{prop:hypernet_scale}
Let $\mathbf{u}_{\omega,\rho}\colon \{0,1\}^M \to \mathbb{R}^{d'}$ be a presence-only token-wise sum-pooled set encoder defined by
\begin{equation*}
    \mathbf{u}_{\omega,\rho}(\mathbf{I}_S) \;=\; f_\rho\!\left(\sum_{i=1}^{M} (\mathbf{I}_S)_i\, f_\omega(\mathbf{e}^1_i)\right) \;=\; f_\rho\!\left(\sum_{i \in S} f_\omega(\mathbf{e}^1_i)\right),
\end{equation*}
and let $g_\phi\colon \mathbb{R}^{d'} \to \mathbb{R}^{|\theta|}$ be a hypernetwork. We assume $f_\omega$, $f_\rho$, and $g_\phi$ are linear at initialisation, in the spirit of variance-propagation analyses~\citep{glorot2010understanding}. Assume the feature embedding pairs $\{(\mathbf{e}^1_i, \mathbf{e}^0_i)\}_{i=1}^M$ are independent across $i$, with each embedding zero-mean, per-coordinate variance $\sigma_e^2$, and finite fourth moments. Then:
\begin{enumerate}
\item[(i)] \textbf{Presence-only encoder, no normalisation.} Let $\theta_S = g_\phi(\mathbf{u}_{\omega,\rho}(\mathbf{I}_S))$. Then
\begin{equation}
  \mathbb{E}\,\|\theta_S\|_2^2 \;=\; |S|\cdot \kappa,
\end{equation}
where $\kappa > 0$ depends on the initialisation scales of $f_\omega$, $f_\rho$, $g_\phi$, and on $\sigma_e^2$, but does not itself depend on $|S|$. The squared norm of the generated weights grows linearly with the number of observed features.

\item[(ii)] \textbf{Presence-only encoder, $L_2$ normalisation.} Let $\tilde{\mathbf{u}}_{\omega,\rho}(\mathbf{I}_S) = \mathbf{u}_{\omega,\rho}(\mathbf{I}_S) / \|\mathbf{u}_{\omega,\rho}(\mathbf{I}_S)\|_2$ and $\theta_S = g_\phi(\tilde{\mathbf{u}}_{\omega,\rho}(\mathbf{I}_S))$. Then
\begin{equation}
  \mathbb{E}\,\|\theta_S\|_2^2 \;=\; \kappa',
\end{equation}
for a constant $\kappa' > 0$ independent of $|S|$.

\item[(iii)] \textbf{Encoder with absence embeddings.} Let $\mathbf{z}_{\omega,\rho}\colon \{0,1\}^M \to \mathbb{R}^{d'}$ be the encoder used in \textsc{Hyper-DFS} (Eq.~\eqref{eq:subset_encoding}),
\begin{equation*}
    \mathbf{z}_{\omega,\rho}(\mathbf{I}_S) \;=\; f_\rho\!\left(\sum_{i=1}^{M} f_\omega(\mathbf{t}_i)\right),
    \qquad
    \mathbf{t}_i = \mathbf{I}_{S,i}\,\mathbf{e}^1_i + \bigl(1 - \mathbf{I}_{S,i}\bigr)\,\mathbf{e}^0_i,
\end{equation*}
and let $\theta_S = g_\phi(\mathbf{z}_{\omega,\rho}(\mathbf{I}_S))$. Then
\begin{equation}
  \mathbb{E}\,\|\theta_S\|_2^2 \;=\; M \cdot \kappa'',
\end{equation}
for a constant $\kappa'' > 0$ independent of $|S|$.
\end{enumerate}
\end{proposition}

\begin{proof}
Since $f_\omega$, $f_\rho$, and $g_\phi$ are linear at initialisation, $\theta_S$ is a linear function of the embeddings, and each linear map preserves the squared-norm scaling up to a multiplicative constant independent of $|S|$. It therefore suffices to track $\mathbb{E}\|\mathbf{u}_{\omega,\rho}(\mathbf{I}_S)\|_2^2$ in part (i), $\mathbb{E}\|\tilde{\mathbf{u}}_{\omega,\rho}(\mathbf{I}_S)\|_2^2$ in part (ii), and $\mathbb{E}\|\mathbf{z}_{\omega,\rho}(\mathbf{I}_S)\|_2^2$ in part (iii).

\emph{Part (i).} Since $f_\omega$ is linear and the $\{\mathbf{e}^1_i\}_{i \in S}$ are independent and zero-mean, so are the summands $\{f_\omega(\mathbf{e}^1_i)\}_{i \in S}$. The cross terms in the expansion of the squared norm therefore vanish, giving
\begin{equation}
    \mathbb{E}\!\left\|\sum_{i \in S} f_\omega(\mathbf{e}^1_i)\right\|_2^2
    \;=\; \sum_{i \in S} \mathbb{E}\|f_\omega(\mathbf{e}^1_i)\|_2^2
    \;=\; |S| \cdot c_1,
\end{equation}
with $c_1 = \mathbb{E}\|f_\omega(\mathbf{e}^1_1)\|_2^2$. Applying the linear $f_\rho$ and $g_\phi$ preserves the $|S|$-scaling up to a constant, giving $\mathbb{E}\|\theta_S\|_2^2 = |S| \cdot \kappa$.

\emph{Part (ii).} By construction, $\|\tilde{\mathbf{u}}_{\omega,\rho}(\mathbf{I}_S)\|_2 = 1$, so $\mathbb{E}\|\tilde{\mathbf{u}}_{\omega,\rho}(\mathbf{I}_S)\|_2^2 = 1$ independently from $|S|$. Hence $\mathbb{E}\|\theta_S\|_2^2 = \kappa'$ for some constant $\kappa'$ independent of $|S|$.

\emph{Part (iii).} The token sum has $M$ summands, with $\mathbf{t}_i = \mathbf{e}^1_i$ for $i \in S$ and $\mathbf{t}_i = \mathbf{e}^0_i$ for $i \in \bar{S}$. Independence and zero-mean of the embedding pairs across $i$ give
\begin{equation} \label{eq:absence_proof}
    \mathbb{E}\!\left\|\sum_{i=1}^M f_\omega(\mathbf{t}_i)\right\|_2^2
    \;=\; \sum_{i \in S} \mathbb{E}\|f_\omega(\mathbf{e}^1_i)\|_2^2 + \sum_{i \in \bar{S}} \mathbb{E}\|f_\omega(\mathbf{e}^0_i)\|_2^2
    \;=\; |S| \cdot c_1 + (M - |S|) \cdot c_0,
\end{equation}
with $c_0 = \mathbb{E}\|f_\omega(\mathbf{e}^0_1)\|_2^2$. Since $f_\omega$ is linear and both embedding distributions share the same per-coordinate variance $\sigma_e^2$, we have $c_0 = c_1$, and the sum reduces to $M \cdot c_1$, independent of $|S|$. Applying the linear $f_\rho$ and $g_\phi$ preserves this scaling, giving $\mathbb{E}\|\theta_S\|_2^2 = M \cdot \kappa''$.
\end{proof}

\paragraph{Remark.} The cardinality independence in part~(iii) depends on $c_0 = c_1$, which follows from our assumption that the absence and presence embeddings share the same per-coordinate variance. This holds at initialisation by construction, but as training progresses the embedding distributions can drift apart, and the dependence on $|S|$ may reappear. Part~(ii), in contrast, gives a guarantee that does not rely on this matching condition. This is why \textsc{Hyper-DFS} combines absence embeddings with $L_2$ normalisation: the absence embeddings provide a uniform $M$-token representation, while the $L_2$ normalisation ensures that the scale of $\theta_S$ stays under control even if the embedding distributions diverge.

\subsection{Output-Scale Regularisation}

To encourage variance uniformity throughout training, we augment the objective with a penalty on the scale of the generated weights:
\begin{equation}
    \mathcal{L}_{\mathrm{scale}}(\phi, \omega, \rho)
    \;=\;
    \lambda \cdot \sum_{l=1}^{L}
    \mathbb{E}_S\!\left[
        \left(
            \frac{1}{n_l}
            \bigl\|g_\phi^{(l)}\bigl(\tilde{\mathbf{z}}_{\omega,\rho}(\mathbf{I}_S)\bigr)\bigr\|_2^2
            -\sigma_{l,\star}^2
        \right)^{\!2}
    \right],
    \label{eq:scale_reg}
\end{equation}
where $n_l$ is the number of parameters in layer $l$ of the primary network, $g_\phi^{(l)}(\cdot) \in \mathbb{R}^{n_l}$ is the slice of the hypernetwork output that generates them, and $\sigma_{l,\star}^2$ is the Xavier target variance for that layer. $L$ is the number of layers in the primary network, and the expectation over $S$ is approximated by the mini-batch of subsets sampled alongside the training data. The coefficient $\lambda > 0$ is annealed to zero after an initial warm-up phase; its primary role is to stabilise the early training trajectory. $\mathcal{L}_{\mathrm{scale}}$ is fully differentiable and adds negligible overhead to the main loss computation.

\section{Stabilizing the training of \textsc{Hyper-DFS} while preventing representation collapse}
\label{app:hyperensembling}
This appendix details the two auxiliary regularisers that complement the main cross-entropy objective in our training procedure: noise injection applied to the subset encoding, and a VICReg-style variance penalty that prevents either the encoder or the hypernetwork from degenerating into a constant map.

\subsection{Stochasticity in \textsc{Hyper-DFS} training} \label{sec:stochastic_ensemble}
Stochasticity can be introduced into hypernetworks in several ways. It is common practice to either perturb the conditioning input \citep{krueger2017bayesian} or generate the primary-network weights directly from noise. In \textsc{Hyper-DFS}, we inject noise into the hypernetwork's input, in $\mathbf{z}$-space:
\begin{equation}
    \mathbf{z}_{\omega,\rho}^*(\mathbf{I}_S) =
    \mathbf{z}_{\omega,\rho}(\mathbf{I}_S) + \boldsymbol{\varepsilon},
    \quad \boldsymbol{\varepsilon} \sim \mathcal{N}(\mathbf{0},\, \sigma^2 \mathbf{I}_{d'}),
    \qquad
    \tilde{\mathbf{z}}_{\omega,\rho}^*(\mathbf{I}_S) =
    \frac{\mathbf{z}_{\omega,\rho}^*(\mathbf{I}_S)}
         {\|\mathbf{z}_{\omega,\rho}^*(\mathbf{I}_S)\|_2},
    \label{eq:stochastic_encoding}
\end{equation}
and pass $\tilde{\mathbf{z}}_{\omega,\rho}^*(\mathbf{I}_S)$ to the hypernetwork in place of $\tilde{\mathbf{z}}_{\omega,\rho}(\mathbf{I}_S)$ during training. This encourages the hypernetwork to vary smoothly as a function of $\mathbf{z}_{\omega,\rho}(\mathbf{I}_S)$, aligning representation-space proximity with functional similarity. It also penalises the Jacobian norm of $g_\phi$ with respect to its conditioning input~\citep{bishop1995training, rifai2011contractive}, which discourages large weight changes in response to small encoding perturbations.

\subsection{Preventing Representation Collapse}
A known failure mode of hypernetworks is \emph{representation collapse}: the hypernetwork may learn to produce nearly identical parameters regardless of the conditioning input, rendering the conditioning redundant. In our setting, collapse can occur at two levels: the encoder may map multiple, or even all, functionally distinct subsets to the same encoding, or the hypernetwork may map distinct encodings to the same weight configuration for the primary network.

We address both modes with a low-variance penalty applied to both the intermediate representations and the generated weights:
\begin{equation}
    \mathcal{L}_\mathrm{collapse}
    \;=\;
    -\frac{1}{d'} \sum_{j=1}^{d'}
    \mathrm{Var}\!\left[
        \bigl(\tilde{\mathbf{z}}_{\omega, \rho}(\mathbf{I}_{S})\bigr)_j
    \right]
    \;-\;
    \frac{1}{|\theta|} \sum_{k=1}^{|\theta|}
    \mathrm{Var}\!\left[
        \bigl(g_\phi(\tilde{\mathbf{z}}_{\omega,\rho}(\mathbf{I}_S))\bigr)_k
    \right],
    \label{eq:collapse_loss}
\end{equation}
where $\mathrm{Var}[\cdot]$ denotes the empirical variance computed across the mini-batch.

Note that this loss does not involve predictions, so an uninformative feature may legitimately not change the model's output, but it should still shift $\tilde{\mathbf{z}}_{\omega,\rho}(\mathbf{I}_S)$ and, in turn, $g_\phi(\tilde{\mathbf{z}}_{\omega,\rho}(\mathbf{I}_S))$. Both variances are computed over the same mini-batch already used for the cross-entropy step, so $\mathcal{L}_\mathrm{collapse}$ adds no additional forward passes. This results in the following training
  \begin{equation}
      \mathcal{L}
      = \mathcal{L}_\text{CE}+\lambda_\text{scale}(t)\mathcal{L}_\text{scale}+\lambda_\text{collapse}\mathcal{L}_\text{collapse},
      \label{eq:total_loss}                                        
  \end{equation}
where $\lambda_{\mathrm{scale}}(t)$ is the annealed warm-up coefficient (Eq.~\eqref{eq:scale_reg}) and $\lambda_\text{collapse} = 0.01$ is fixed throughout training.

\section{Gradient Dilution in \textsc{Hyper-DFS} and the Pre-training
Stabilisation Strategy}
\label{apx:grad_dilution}

This appendix develops the training-stability analysis summarised in Section~\ref{sec:grad_dilution}. Conditioning a hypernetwork on different feature subsets within a single mini-batch can induce \emph{gradient dilution}: subset-specific contributions to the backward pass interfere destructively. The effect appears at the hypernetwork parameters $\phi$ and is amplified at the input compressor $c_\eta$. We motivate below the two remedies used during predictor training: a per-batch mask budget and a linear learning-rate warm-up.

\paragraph{Gradient dilution at $\phi$ and $\eta$.}
Consider a mini-batch of size $B$ in which $K$ distinct feature subsets $\{S_1, \ldots, S_K\}$ appear, with $\mathcal{B}_k \subseteq \{1, \ldots, B\}$ collecting the samples that share subset $S_k$. The gradient at the hypernetwork parameters decomposes as
\begin{equation}
    \nabla_\phi \mathcal{L}
    \;=\;
    \frac{1}{B} \sum_{k=1}^{K}
    J_\phi\!\bigl(\tilde{\mathbf{z}}_{\omega,\rho}(\mathbf{I}_{S_k})\bigr)^{\!\top}
    \sum_{b \in \mathcal{B}_k}
    \nabla_\theta \mathcal{L}^{(b)}\,\bigr|_{\theta = \theta_{S_k}},
    \label{eq:grad_phi_apx}
\end{equation}
where $\theta_{S_k} = g_\phi(\tilde{\mathbf{z}}_{\omega,\rho}(\mathbf{I}_{S_k}))$ are the weights generated for subset $S_k$, $J_\phi(\mathbf{z}) = \partial g_\phi(\mathbf{z}) / \partial \phi \in \mathbb{R}^{|\theta| \times |\phi|}$ is the parameter-Jacobian of the weight-generation function, and $\mathcal{L}^{(b)}$ is the loss on the $b$-th sample. Samples sharing a subset accumulate constructively through a common Jacobian. However, $g_\phi$ at initialisation behaves as a near-random function, so the Jacobians for distinct $k$ act as approximately uncorrelated linear maps. Their contributions therefore rarely share a coherent direction, reducing the effective signal in $\nabla_\phi \mathcal{L}$~\citep{yu2020gradient}.

The compressor is mask-independent in its forward pass, but its backward signal flows through the generated, mask-dependent primary network:
\begin{equation}
    \nabla_\eta \mathcal{L}
    \;=\;
    \frac{1}{B} \sum_{b=1}^{B}
    J_\eta\!\bigl(\tilde{\mathbf{x}}^{(b)}_{S^{(b)}}\bigr)^{\!\top}
    J_{f,\theta}\!\bigl(\mathbf{h}^{(b)}\bigr)^{\!\top}
    \nabla_{\hat{y}^{(b)}_\theta} \mathcal{L}^{(b)}
    \,\bigr|_{\theta = \theta^{(b)}},
    \label{eq:grad_compressor_apx}
\end{equation}
where $S^{(b)}$ is the subset assigned to the $b$-th sample, $\theta^{(b)} = g_\phi(\tilde{\mathbf{z}}_{\omega,\rho}(\mathbf{I}_{S^{(b)}}))$ is the per-sample generated weight tensor, $\mathbf{h}^{(b)} = c_\eta(\tilde{\mathbf{x}}^{(b)}_{S^{(b)}})$ is the compressed representation, $\hat{y}^{(b)}_\theta = f_\theta(\mathbf{h}^{(b)})$ is the prediction at weights $\theta$, $J_\eta(\mathbf{x}) = \partial c_\eta(\mathbf{x}) / \partial \eta$ is the parameter-Jacobian of the compressor, and $J_{f,\theta}(\mathbf{h}) = \partial f_\theta(\mathbf{h}) / \partial \mathbf{h}$ is the input-Jacobian of the primary network at weights $\theta$. The bar fixes both $J_{f,\theta}$ and $\hat{y}^{(b)}_\theta$ at $\theta = \theta^{(b)}$. At initialisation, distinct subsets produce approximately independent weight configurations, and the corresponding input-Jacobians act as approximately uncorrelated linear maps. This is analogous to training a shared feature extractor behind randomly initialised, rapidly changing projection heads, a scenario prone to representation collapse~\citep{yu2020gradient, chen2021exploring}.

\paragraph{A per-batch mask budget.}
Restricting each mini-batch to a single subset ($K = 1$, with $S^{(b)} = S$ for all $b$) gives $\theta^{(b)} = \theta_S$ for all $b$, and Eq.~\eqref{eq:grad_compressor_apx} simplifies to
\begin{equation}
    \nabla_\eta \mathcal{L}
    \;=\;
    \frac{1}{B} \sum_{b=1}^{B}
    J_\eta\!\bigl(\tilde{\mathbf{x}}^{(b)}_S\bigr)^{\!\top}
    J_{f,\theta}\!\bigl(\mathbf{h}^{(b)}\bigr)^{\!\top}
    \nabla_{\hat{y}^{(b)}_\theta} \mathcal{L}^{(b)}
    \,\bigr|_{\theta = \theta_S}.
    \label{eq:grad_compressor_K1}
\end{equation}
Every per-sample contribution is now routed through the same primary network, recovering the standard situation of a feature extractor trained jointly with a single downstream classifier; the same simplification applies to Eq.~\eqref{eq:grad_phi_apx}, where the outer sum collapses to the single term $k=1$ and the bar fixes $\theta = \theta_S$. A strict budget of $K = 1$ removes the within-batch subset diversity needed by the variance penalty $\mathcal{L}_{\mathrm{collapse}}$ (Eq.~\eqref{eq:collapse_loss}) and forfeits the cooperation and generalisation effects of intra-batch diversity, so we adopt a small but non-trivial $K$ in practice.

\paragraph{Linear learning-rate warm-up.}
Even with a small $K$, a transient instability persists. The scale regulariser $\mathcal{L}_{\mathrm{scale}}$ (Eq.~\eqref{eq:scale_reg}) requires several optimiser steps to bring the empirical layerwise mean squared values $\|g_\phi^{(l)}(\tilde{\mathbf{z}}_{\omega,\rho}(\mathbf{I}_S))\|_2^2 / n_l$ towards their Xavier targets $\sigma_{l,\star}^2$. During this transient, a large optimiser step can push $\theta_S$ into a saturated activation regime, suppressing $J_{f,\theta_S}(\mathbf{h})$ and the compressor gradient itself. Linearly warming up the global learning rate over the first $T_{\mathrm{warm}}$ steps~\citep{goyal2017accurate} keeps these early updates small, allowing $g_\phi$ to reach the regime in which $\mathcal{L}_{\mathrm{scale}}$ takes effect before the optimiser commits to a particular mapping $\tilde{\mathbf{z}}_{\omega,\rho}(\mathbf{I}_S) \mapsto \theta_S$.

\section{Additional Experimental Results}
\label{apx:additional_results}

In this appendix we cover additional experimental results. We report the results of our ablation studies, extended reports of the results reported in the main body of the paper, and we report the running times of the different methods tested.

\subsection{Ablation Studies}

  \label{app:ablations}                                     

  To isolate the contribution of each design choice in \textsc{Hyper-DFS}, we conduct a series of four controlled ablations, varying one component at a time while holding all others at their default configuration.

    \begin{table}[t]  
    \centering    
    \caption{Effect of changing the hypernetwork for a standard MLP in the performance of \textsc{Hyper-DFS}, measured using AUAC-F1.}     
    \label{tab:abl_hypernetwork}  
    \adjustbox{width=\linewidth}{%
    \begin{tabular}{lcccccccccc}  
      \toprule    
      Variant & Bank & California & Cirrhosis & Diabetes & Heart & Metabric & MiniBooNE & Wine & Yeast & Overall \\
      \midrule    
      Hyper-DFS (full)
& \textbf{80.35} & 82.65 & \textbf{49.22} & \textbf{59.34} & 82.09     
& \textbf{66.96} & \textbf{85.54} & 93.67 & \textbf{41.82} & \textbf{71.29} \\ 
      {} & {\scriptsize\textcolor{gray}{$\pm$0.48}}    
& {\scriptsize\textcolor{gray}{$\pm$1.13}}    
& {\scriptsize\textcolor{gray}{$\pm$3.38}}
& {\scriptsize\textcolor{gray}{$\pm$1.77}}    
& {\scriptsize\textcolor{gray}{$\pm$3.86}}    
& {\scriptsize\textcolor{gray}{$\pm$1.23}}
& {\scriptsize\textcolor{gray}{$\pm$0.25}}    
& {\scriptsize\textcolor{gray}{$\pm$2.10}}
& {\scriptsize\textcolor{gray}{$\pm$2.25}}    
& {\scriptsize\textcolor{gray}{$\pm$16.97}} \\
      w/o hypernetwork
& 78.80 & \textbf{83.14} & 48.77 & 57.46 & \textbf{82.10}
& 62.13 & 83.75 & \textbf{94.59} & 39.28 & 70.00 \\   
      {}      
& {\scriptsize\textcolor{gray}{$\pm$0.48}}    
& {\scriptsize\textcolor{gray}{$\pm$0.55}}    
& {\scriptsize\textcolor{gray}{$\pm$3.12}}    
& {\scriptsize\textcolor{gray}{$\pm$3.70}}
& {\scriptsize\textcolor{gray}{$\pm$4.83}}    
& {\scriptsize\textcolor{gray}{$\pm$7.34}}
& {\scriptsize\textcolor{gray}{$\pm$0.45}}    
& {\scriptsize\textcolor{gray}{$\pm$3.94}}
& {\scriptsize\textcolor{gray}{$\pm$1.87}}    
& {\scriptsize\textcolor{gray}{$\pm$18.02}} \\
      \bottomrule     
    \end{tabular} 
    }%
  \end{table}

\paragraph{Hypernetwork ablation.} Table \ref{tab:abl_hypernetwork} reports the performance degradation when the hypernetwork is replaced by a single MLP with mask concatenation, holding the Set Transformer encoding fixed. \textsc{Hyper-DFS} improves AUAC-F1 on 6 of 9 datasets and the three exceptions lie within one standard deviation. The effect is most pronounced on Metabric, where the MLP baseline is 4.8 points worse on average and more variable across folds (std 7.34 vs 1.23), indicating that per-subset specialisation stabilises training in the high-dimensional regime. The largest gains concentrate on high-M / many-class datasets (Metabric, MiniBooNE, Diabetes, Yeast), which makes the hypernetwork especially interesting for hard DFS problems.

\begin{table}[t]
	\centering
	\caption{Improvement of the learned DFS policy compared to a random policy, measured using AUAC-F1.}
	\label{tab:delta_auac_f1}
		\begin{tabular}{
                l
                *{4}{c}
                @{\hspace{6pt}}
                *{3}{c}
                @{\hspace{6pt}}
                c
            }
			\toprule
			& \multicolumn{4}{c}{Single-model} & \multicolumn{3}{c}{Multi-model} & \multicolumn{1}{c}{Hypernetworks} \\
			\cmidrule(lr){2-5}\cmidrule(lr){6-8}\cmidrule(lr){9-9}
			{Dataset} & {CWCF} & {RePa} & {SEFA} & {VIP} & {Card.} & {Ensemble} & {MoE} & {Hyper-DFS} \\
			\midrule
			Bank & 0.78 & 11.86 & 11.85 & {\bfseries\underline{13.31}} & 2.59 & 11.85 & \bfseries 11.96 &  12.31 \\
			California & 2.46 & 6.81 & 6.52 & {\bfseries\underline{7.79}} & -1.15 & \bfseries 7.17 & 5.38 &  4.42 \\
			Cirrhosis & 1.90 & 4.41 & {\bfseries\underline{6.53}} & 5.02 & 0.22 & 4.53 & \bfseries 4.78 & 2.64 \\
			Diabetes & 0.71 & 9.69 & 20.50 & \bfseries 21.85 & 12.20 & 17.32 & \bfseries 19.76 & {\bfseries\underline{22.48}} \\
			Heart & 1.60 & 5.80 & 5.63 & {\bfseries\underline{6.30}} & 1.38 & 5.61 & \bfseries 5.95 &  4.81 \\
			Metabric & 2.14 & \bfseries 20.23 & 14.39 & 17.85 & 14.06 & {\bfseries\underline{22.15}} & 16.43 &  18.98 \\
			Miniboone & 3.52 & 20.08 & 21.75 & {\bfseries\underline{24.42}} & 16.49 & \bfseries 22.66 & 21.49 &  16.84 \\
			Wine & 6.35 & 7.11 & 6.82 & \bfseries 7.19 & 4.30 & 7.69 & {\bfseries\underline{8.24}} &  7.61 \\
			Yeast & -0.71 & \bfseries 5.64 & 4.10 & 1.87 & -3.81 & 4.77 & {\bfseries\underline{6.70}} &  3.42 \\
			\midrule
			Avg Rank & 7.22 & 4.00 & 3.89 & {\bfseries\underline{2.33}} & 7.11 & 3.11 &  \bfseries 2.89 & 4.11 \\
			Avg $\Delta$ & 2.19 & 10.18 & 10.90 & {\bfseries\underline{11.73}} & 5.14 &  \bfseries 11.53 & 10.74 &  10.39 \\
			\bottomrule
	\end{tabular}
\end{table}
\begin{wrapfigure}{r}{0.3\linewidth}
    \centering
    \includegraphics[width=\linewidth]{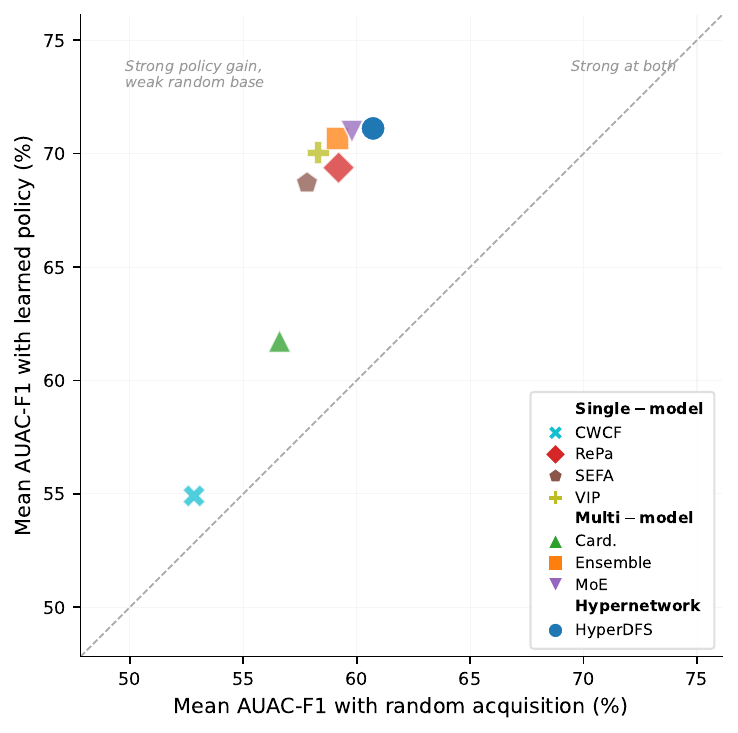}
    \caption{Mean AUAC-F1 averaged across nine tabular benchmarks under random feature acquisition (x-axis) versus learned acquisition policy (y-axis).}
    \label{fig:policy_gains_tab}
\end{wrapfigure}
\paragraph{Acquisition policy.} We ablate the DFS acquisition policy by replacing it with random sequential feature acquisition. The gain $\Delta\text{AUAC} = \text{AUAC}_\text{policy} - \text{AUAC}_\text{random}$ quantifies how much of the performance is attributable to the policy itself, as opposed to the model's general ability to handle arbitrary subsets. Results for tabular datasets are shown in Table \ref{tab:delta_auac_f1}. We can see that the magnitude of the improvement is, on average, around 10 points. Two methods are significantly below this value: CWCF and Cardinality-based routing, which also underperformed several other methods. VIP policy improved performance the most over random acquisition. The relationship between both learned and random DFS policies is shown in Figure \ref{fig:policy_gains_tab}, where we can see that the performance of \textsc{Hyper-DFS} was superior using both random and learned policies, which also means that the \textsc{Hyper-DFS} predictor model is less dependent on the learned acquisition policy to achieve good results.

\begin{figure}
    \centering
    \includegraphics[width=\linewidth]{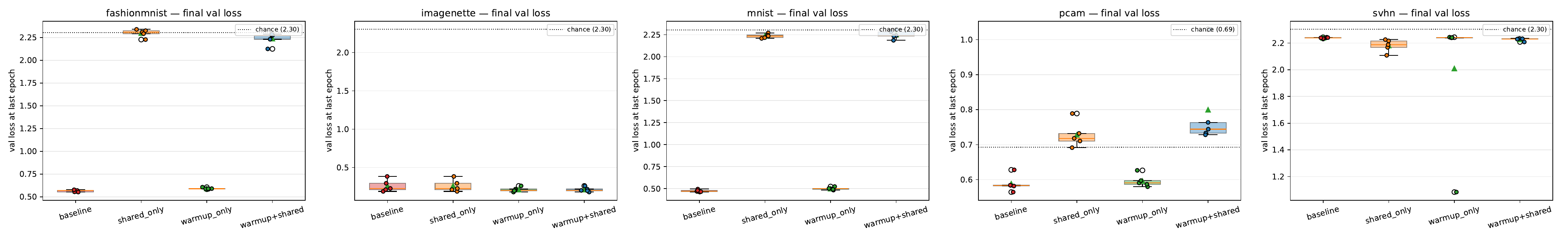}
    \caption{Effect in loss after 10 epochs of the mask restriction and LR warmup. Restricting the number of masks solved the problem of training collapse, but it also hurts performance on the first steps of the training.}
    \label{fig:performance_mask_restriction}
\end{figure}
\begin{table}[t]
  \centering
  \caption{LR Warmup and shared-mask ablation: failed training rate and training improvement as percentage of the initial loss. Metrics reported after 10 epochs and averaged across 5 seeds.}
  \label{tab:warmup-shared-ablation}
  \begin{tabular}{llcc}
    \toprule
    Dataset & Config & Failed & Train improvement \% (mean $\pm$ std) \\
    \midrule
    fashionmnist & baseline & 0/5 & $47.7 \pm 0.6$ \\
     & shared\_only & 0/5 & $55.4 \pm 1.3$ \\
     & warmup+shared & 0/5 & $66.7 \pm 1.3$ \\
     & warmup\_only & 0/5 & $72.7 \pm 0.2$ \\
    \midrule
    imagenette & baseline & 0/5 & $88.5 \pm 1.5$ \\
     & shared\_only & 0/5 & $88.5 \pm 1.5$ \\
     & warmup+shared & 0/5 & $95.3 \pm 0.3$ \\
     & warmup\_only & 0/5 & $95.3 \pm 0.3$ \\
    \midrule
    mnist & baseline & 0/5 & $58.4 \pm 0.7$ \\
     & shared\_only & 0/5 & $64.0 \pm 1.7$ \\
     & warmup+shared & 0/5 & $74.7 \pm 1.1$ \\
     & warmup\_only & 0/5 & $78.5 \pm 0.2$ \\
    \midrule
    pcam & baseline & 0/5 & $12.7 \pm 2.0$ \\
     & shared\_only & 0/5 & $11.4 \pm 2.5$ \\
     & warmup+shared & 0/5 & $8.8 \pm 0.7$ \\
     & warmup\_only & 0/5 & $14.3 \pm 0.9$ \\
    \midrule
    svhn & baseline & 5/5 & $0.2 \pm 0.0$ \\
     & shared\_only & 0/5 & $55.7 \pm 1.5$ \\
     & warmup+shared & 0/5 & $50.7 \pm 0.9$ \\
     & warmup\_only & 4/5 & $10.8 \pm 20.6$ \\
    \bottomrule
  \end{tabular}
\end{table}
\paragraph{No mask restriction.} The gradient issue of the compressor loss described in Section 4.2 was solved by restricting the number of different masks in the pretraining to one, and by setting a linear LR Warmup. For this ablation, we see the ratio of failed trainings for all the image datasets tested, shown in Table \ref{tab:warmup-shared-ablation}, after 10 epochs of training. We can see there that SVHN systematically fails to train for every seed tested without LR Warmup or mask restriction, while other datasets run fine. We can also see that LR Warmup solved the problem on one of five seeds, so restricting the number of masks is necessary at the beginning of training. However, we also report in Figure \ref{fig:performance_mask_restriction} that restricting the number of masks also hurts generalisation, suggesting that a compromise between setting the number of different masks to $1$ and not controlling them might be the optimal solution to solve the training stability issues while retaining good performance in the early steps of training.

\subsection{Extended Results} \label{apx:curves}

\begin{figure}[p]
    \centering
    \renewcommand{\arraystretch}{1.2}
    \begin{tabular}{@{}ccc@{}}
        \multicolumn{3}{c}{\textsc{Synthetic datasets}} \\[0.3em]
        \includegraphics[width=0.32\linewidth]{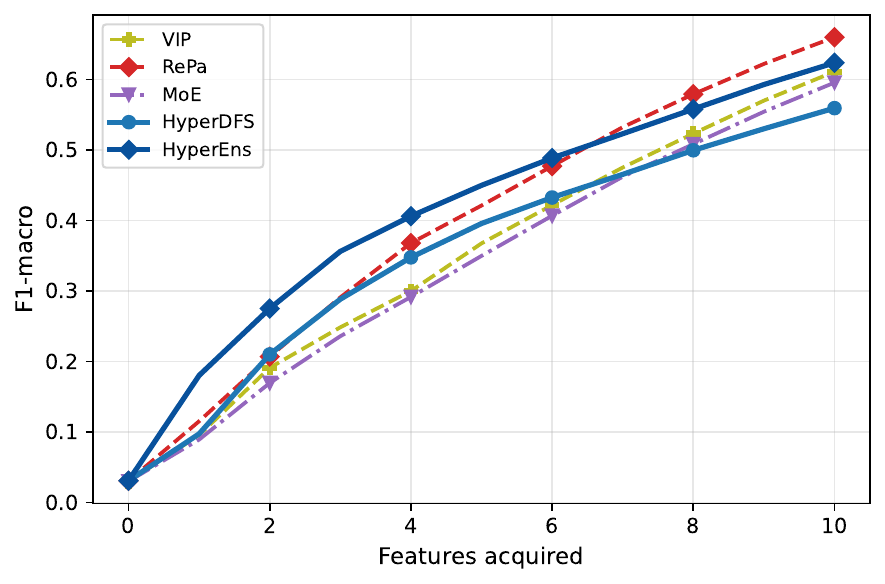} &
        \includegraphics[width=0.32\linewidth]{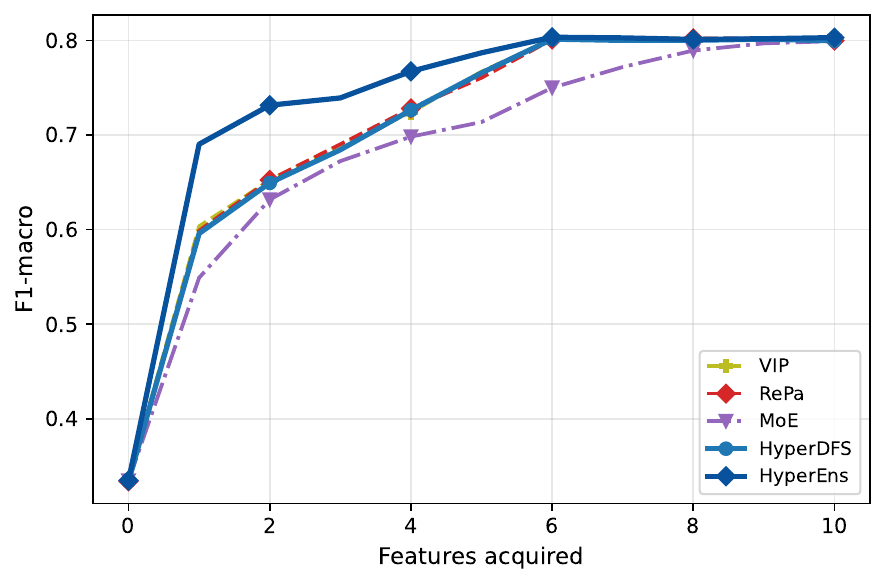} &
        \includegraphics[width=0.32\linewidth]{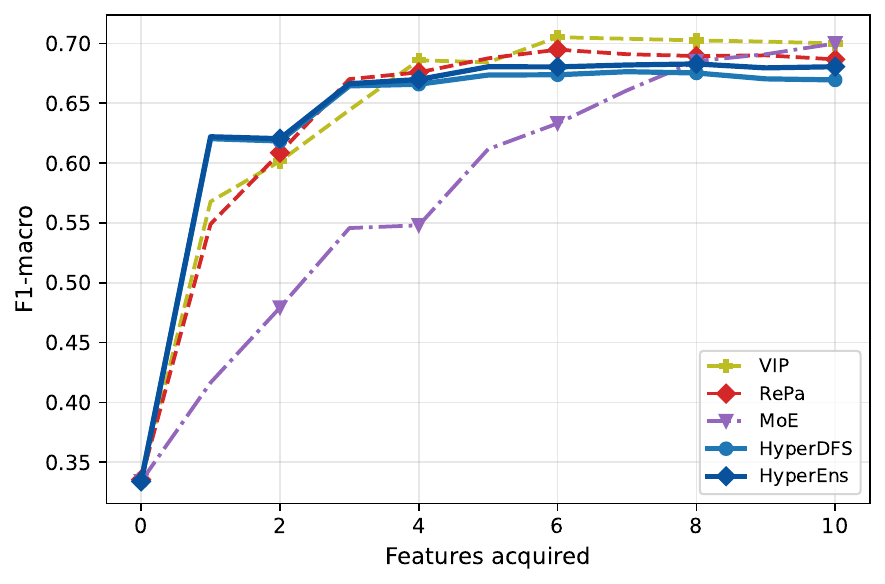} \\[-0.3em]
        {\small (a) Cube} & {\small (b) Sim1} & {\small (c) Sim2} \\[0.5em]
        \includegraphics[width=0.32\linewidth]{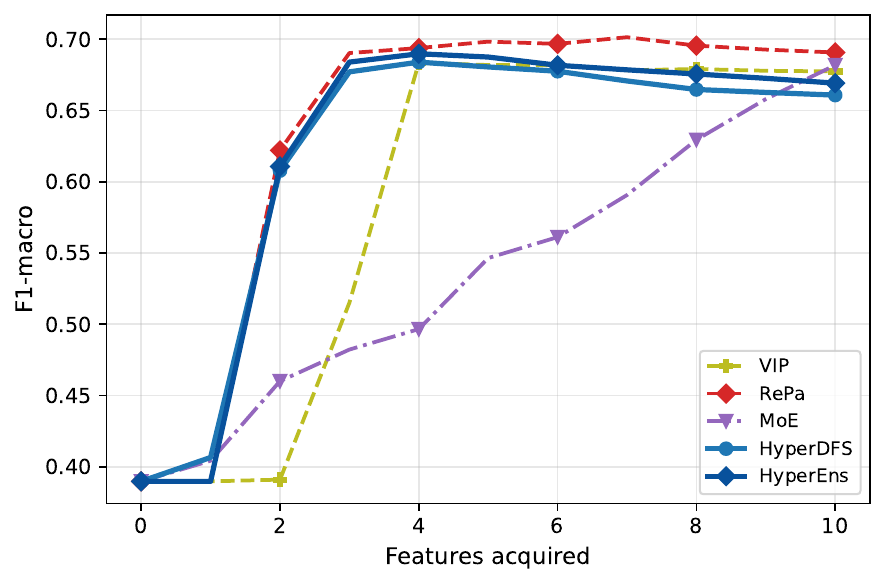} &
        \includegraphics[width=0.32\linewidth]{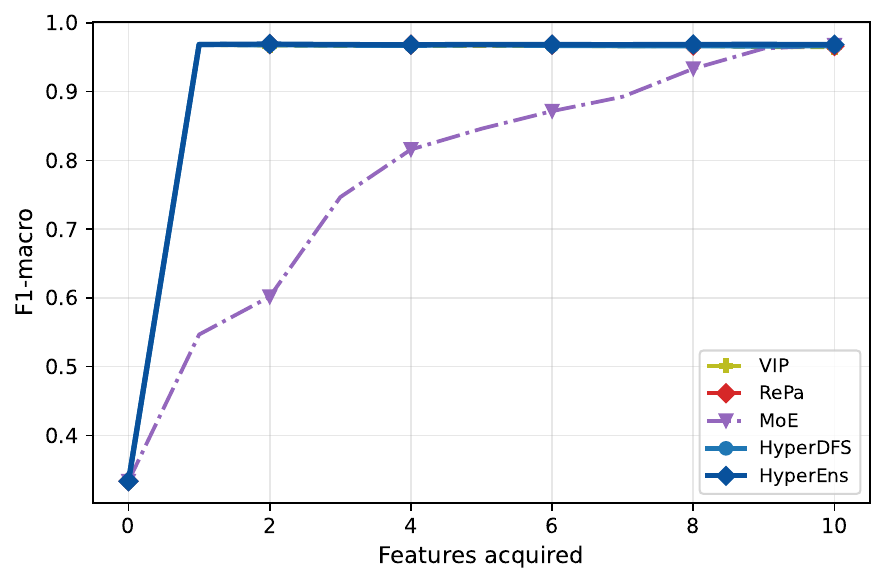} &
        \includegraphics[width=0.32\linewidth]{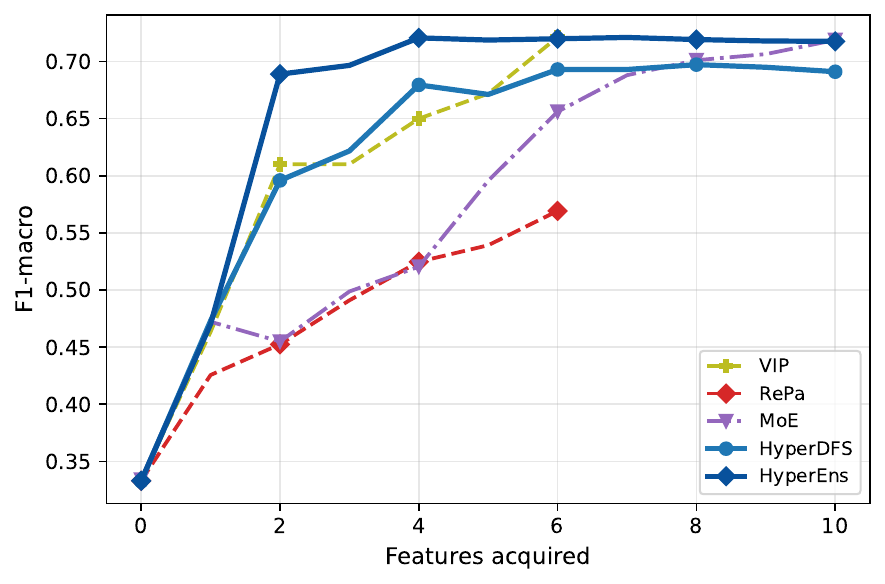} \\[-0.3em]
        {\small (d) Sim3} & {\small (e) ProxySub} & {\small (f) SynPairs} \\[0.8em]
        \midrule \\[-0.6em]
        \multicolumn{3}{c}{\textsc{Tabular datasets}} \\[0.3em]
        \includegraphics[width=0.32\linewidth]{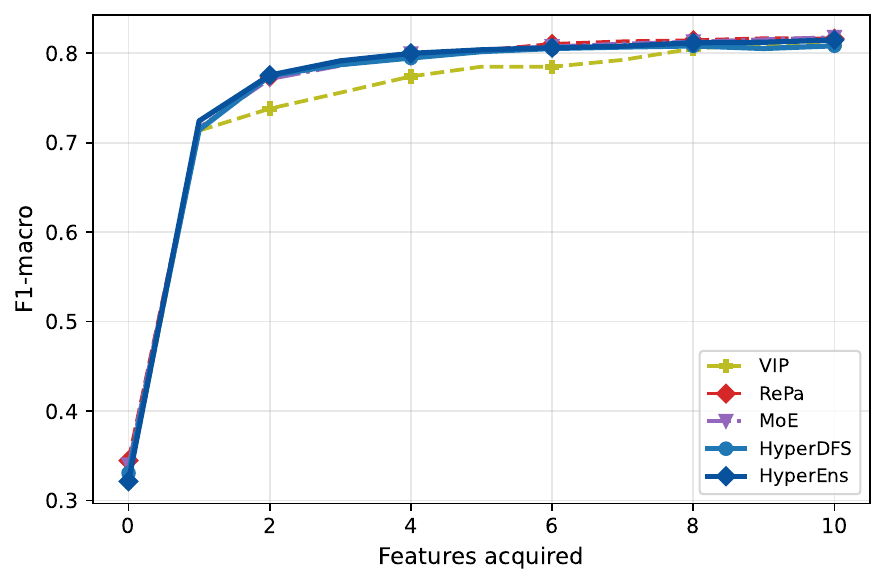} &
        \includegraphics[width=0.32\linewidth]{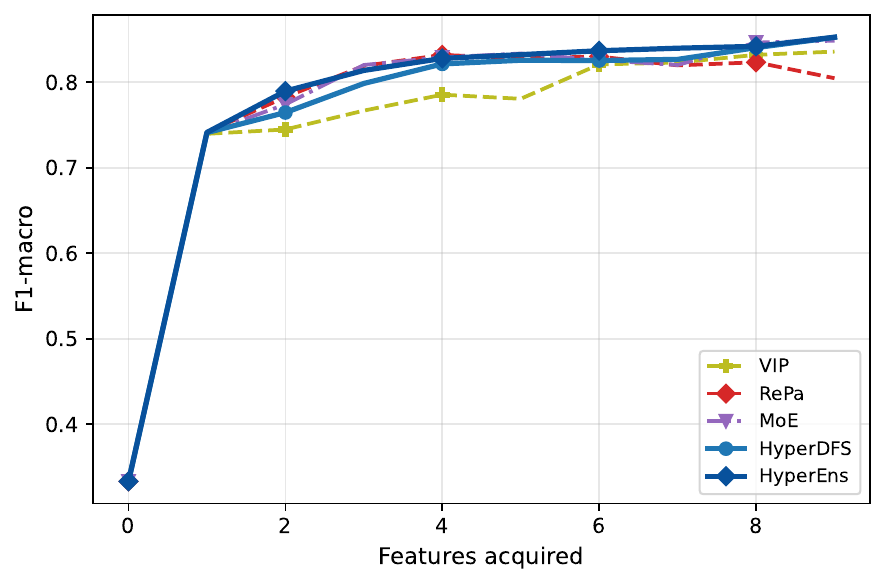} &
        \includegraphics[width=0.32\linewidth]{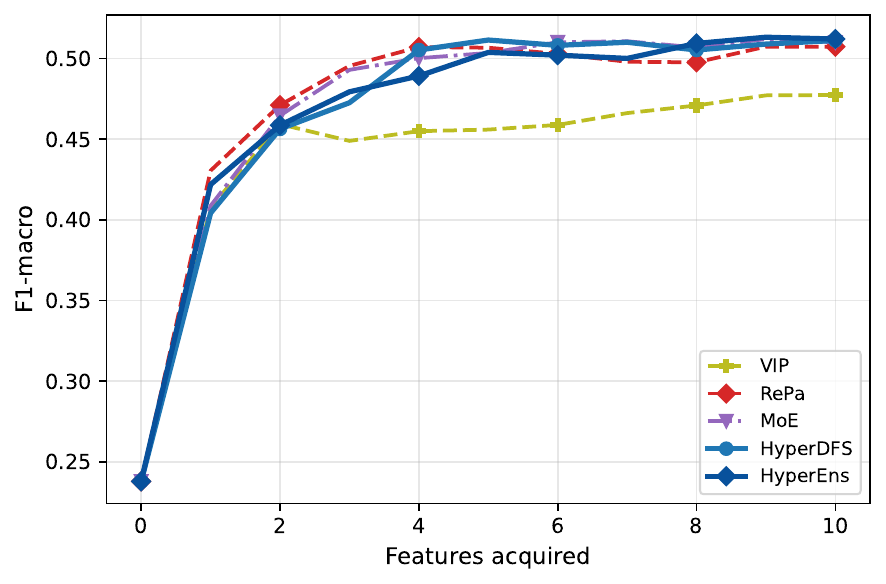} \\[-0.3em]
        {\small (g) Bank} & {\small (h) California} & {\small (i) Cirrhosis} \\[0.5em]
          \includegraphics[width=0.32\linewidth]{figures/best_group_diabetes_f1_macro.pdf} &
        \includegraphics[width=0.32\linewidth]{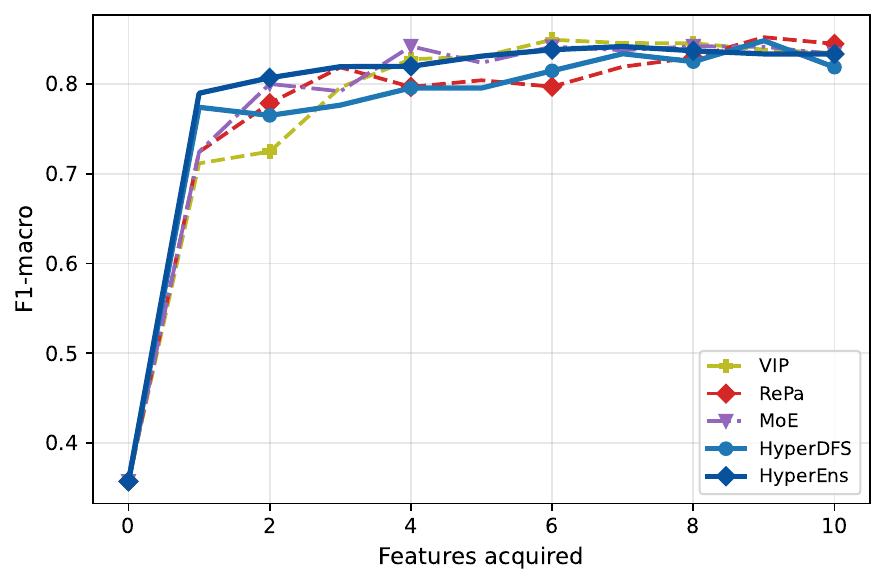} &
        \includegraphics[width=0.32\linewidth]{figures/best_group_metabric_f1_macro.pdf} \\[-0.3em]
        {\small (j) Diabetes} & {\small (k) Heart} & {\small (l) Metabric} \\[-0.3em]
        \includegraphics[width=0.32\linewidth]{figures/best_group_miniboone_f1_macro.pdf} &
        \includegraphics[width=0.32\linewidth]{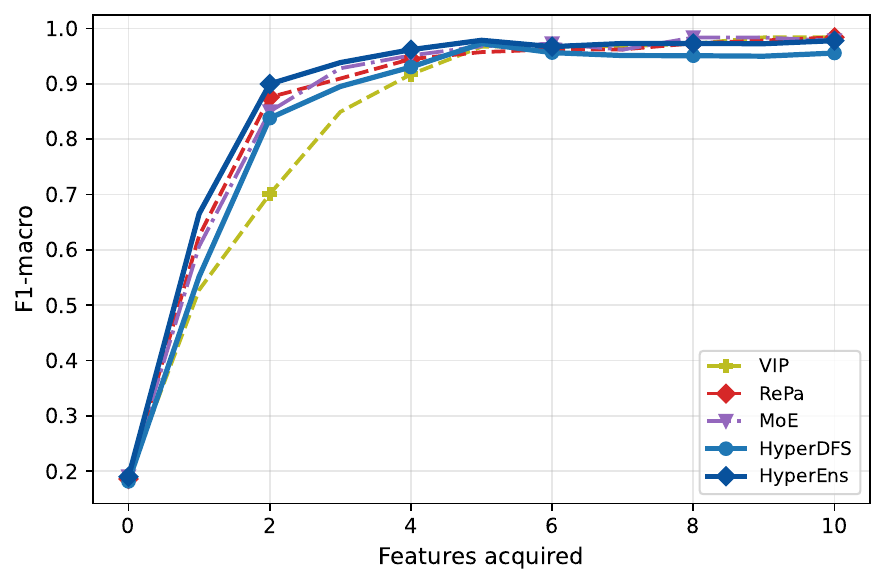} &
        \includegraphics[width=0.32\linewidth]{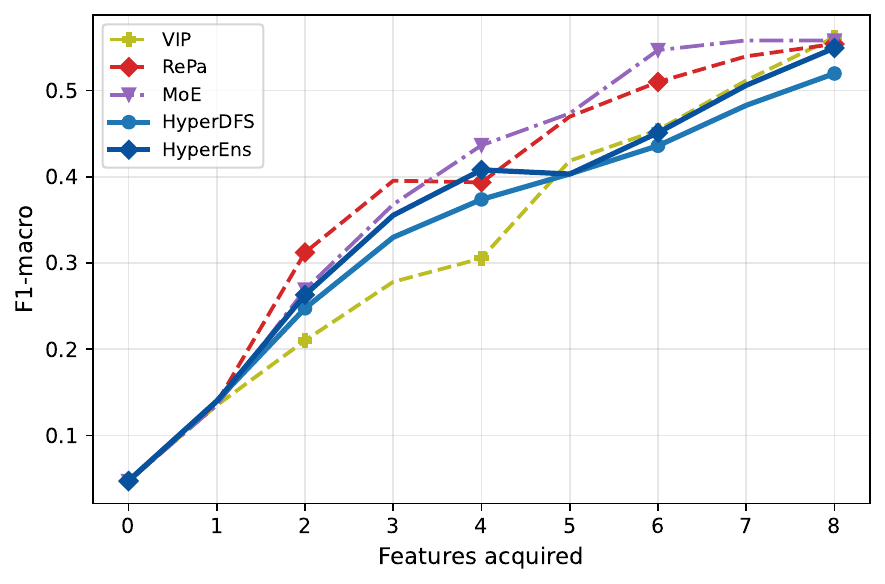} \\[-0.3em]
        {\small (m) Miniboone} & {\small (n) Wine} & {\small (o) Yeast} \\
    \end{tabular}
    \caption{F1-macro acquisition curves for all synthetic (a--f) and tabular (g--l) datasets. Each curve reports the mean F1-macro across 5-fold CV as a function of the number of acquired features, for the top-performing methods.}
    \label{fig:appendix_curves}
\end{figure}

Figure~\ref{fig:appendix_curves} reports the full set of acquisition curves for synthetic and real-world tabular data omitted from the main text. We can see there consistent patterns across datasets. On the synthetic benchmarks, HyperEns achieves the steepest initial gains on Cube, Sim1, and SynPairs, confirming that per-subset specialisation is most valuable when few features have been acquired and the subset space is most heterogeneous. On ProxySub, all methods except MoE reach near-perfect performance with only two features, reflecting the low intrinsic difficulty of this task; the large gap for MoE suggests that learned routing struggles when the discriminative signal is concentrated in very few features. On Sim2 and Sim3, the hypernetwork variants lead at low budgets but are matched or slightly overtaken by RePa and VIP at higher budgets, consistent with the Miniboone pattern discussed in the main text: as the observed set grows, per-subset specialisation provides diminishing returns.
Among the tabular datasets, Bank, California, and Wine exhibit rapid saturation where all top methods converge to similar performance beyond 4--6 features, leaving limited room for differentiation. Heart and SynPairs show the clearest advantage for HyperEns at low-to-mid budgets, while Cirrhosis and Yeast are the most challenging settings, with all methods clustered within narrow performance bands and no single approach consistently dominating. Notably, on Yeast, RePa and MoE achieve the best high-budget performance, suggesting that on tasks with many weakly informative features, the reparametrisation and routing strategies can be more effective than the hypernetwork when the full feature set is nearly observed. Across all datasets, a recurring theme is that the hypernetwork advantage is concentrated in the low-budget regime — precisely the regime of greatest practical interest in cost-sensitive acquisition.

  \begin{figure}[p]
      \centering
      \renewcommand{\arraystretch}{1.2}
      \begin{tabular}{@{}ccc@{}}
          \multicolumn{3}{c}{\textsc{Image datasets}} \\[0.3em]
          \includegraphics[width=0.32\linewidth]{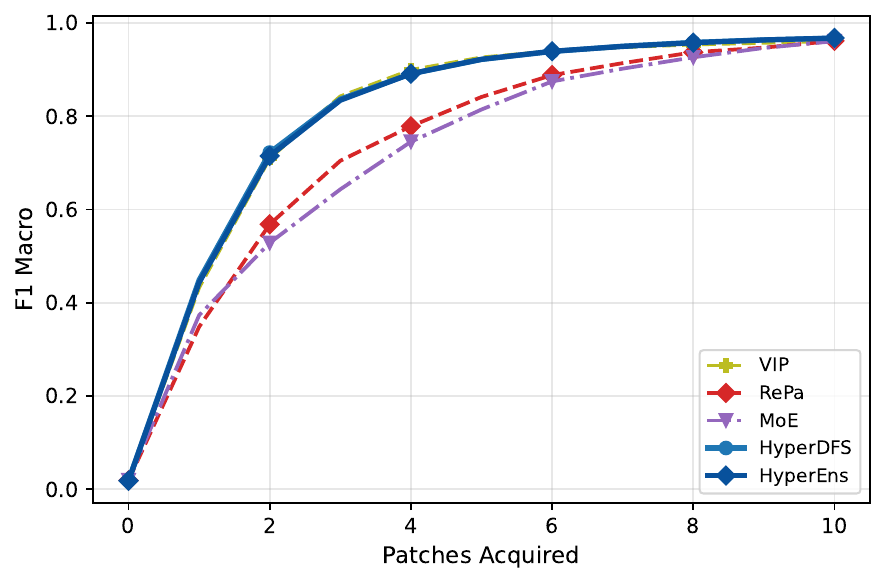} &
          \includegraphics[width=0.32\linewidth]{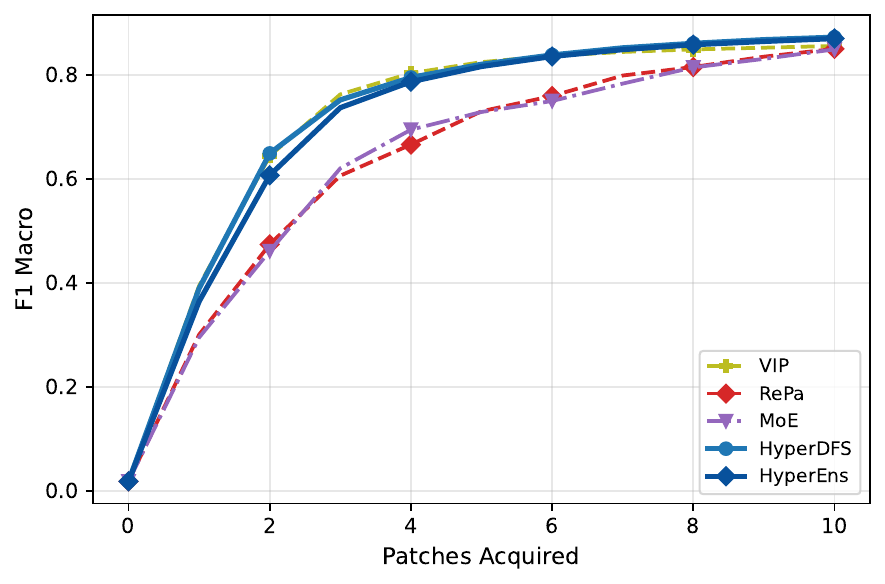} &
          \includegraphics[width=0.32\linewidth]{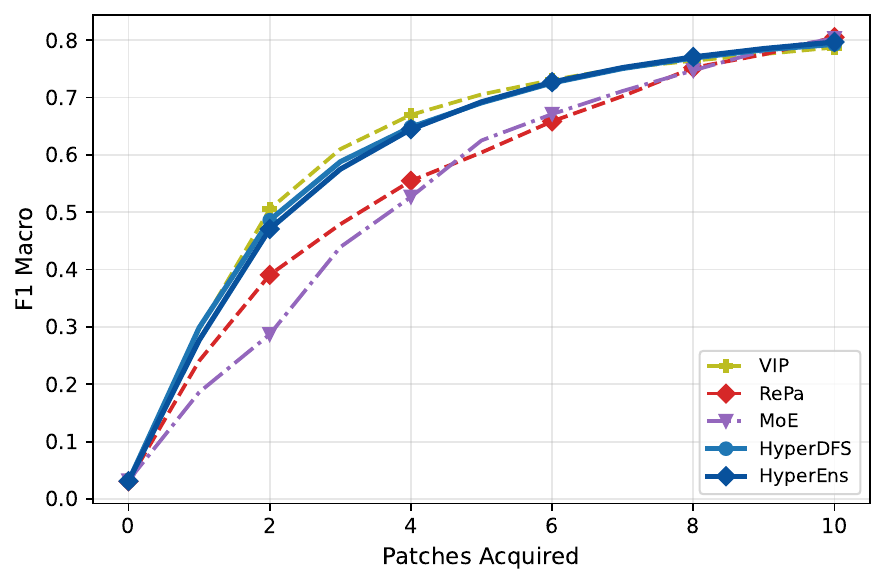} \\[-0.3em]
          {\small (a) MNIST} & {\small (b) Fashion-MNIST} & {\small (c) SVHN} \\[0.5em]
          \includegraphics[width=0.32\linewidth]{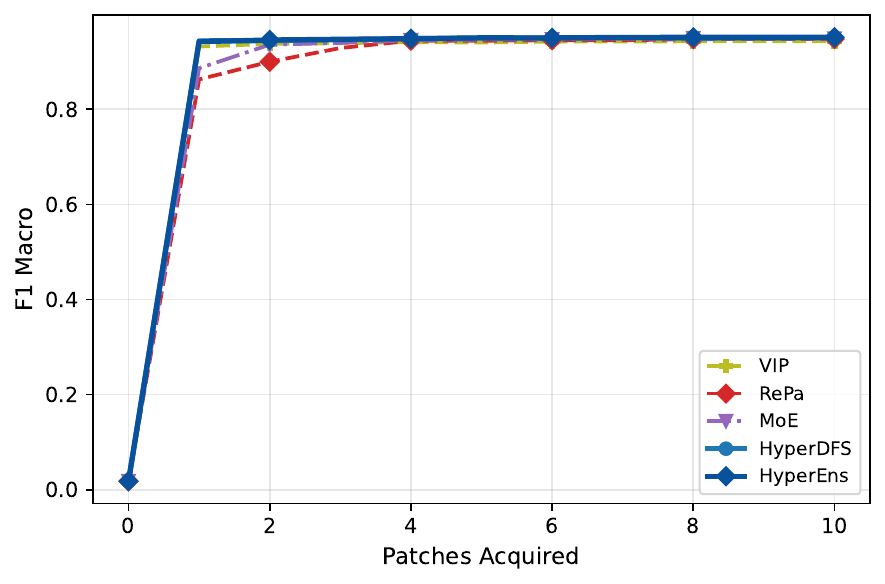} &
          \includegraphics[width=0.32\linewidth]{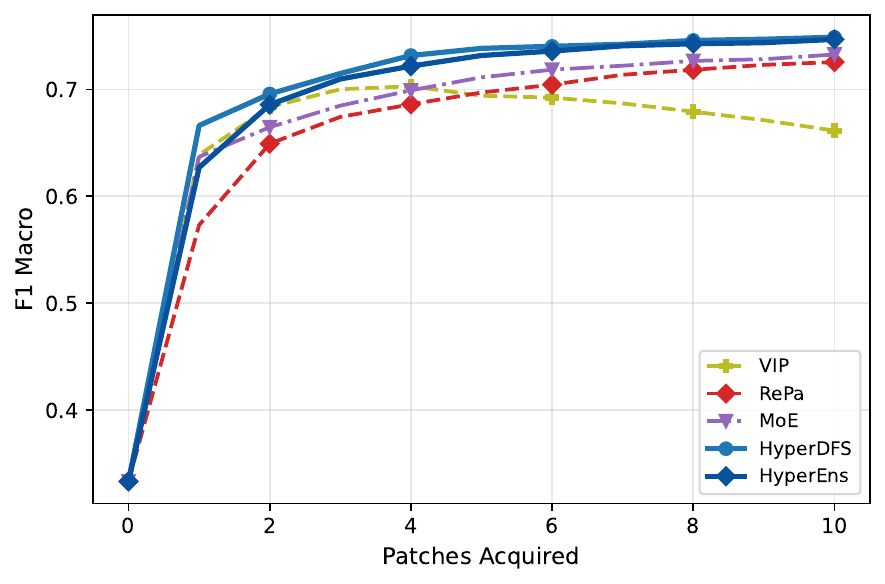} &
          \\[-0.3em]
          {\small (d) Imagenette} & {\small (e) PCam} & \\
      \end{tabular}
      \caption{F1-macro acquisition curves for all image datasets for unseen feature subsets in training. Each curve reports the mean F1-macro across 5-fold CV as a function of the number of acquired features, for the top-performing methods.}
      \label{fig:appendix_image_curves}
  \end{figure}

\begin{table}[t]
	\centering
	\caption{Zero-shot AUAC-F1 on held-out feature subsets (mean $\pm$ std over CV folds, \%).}
	\label{tab:zeroshot_auac_f1}
		\begin{tabular}{l *{3}{S[table-format=2.2, detect-weight=true]}}
			\toprule
			{Dataset} & {VIP} & {MoE} & {Hyper-DFS} \\
			\midrule
			Bank & 59.54 & 70.51 & \bfseries 70.79 \\
			{} & {\scriptsize\textcolor{gray}{$\pm$0.63}} & {\scriptsize\textcolor{gray}{$\pm$0.58}} & {\scriptsize\textcolor{gray}{$\pm$0.83}} \\
			California & 47.83 & 75.90 & \bfseries 76.10 \\
			{} & {\scriptsize\textcolor{gray}{$\pm$3.61}} & {\scriptsize\textcolor{gray}{$\pm$0.95}} & {\scriptsize\textcolor{gray}{$\pm$1.18}} \\
			Cirrhosis & 24.51 & \bfseries 47.00 & 46.85 \\
			{} & {\scriptsize\textcolor{gray}{$\pm$0.64}} & {\scriptsize\textcolor{gray}{$\pm$1.97}} & {\scriptsize\textcolor{gray}{$\pm$2.30}} \\
			Diabetes & \bfseries 49.84 & 45.67 & 49.51 \\
			{} & {\scriptsize\textcolor{gray}{$\pm$1.08}} & {\scriptsize\textcolor{gray}{$\pm$1.83}} & {\scriptsize\textcolor{gray}{$\pm$1.48}} \\
			Heart & 44.94 & \bfseries 76.62 & 75.94 \\
			{} & {\scriptsize\textcolor{gray}{$\pm$3.45}} & {\scriptsize\textcolor{gray}{$\pm$4.27}} & {\scriptsize\textcolor{gray}{$\pm$4.55}} \\
			Metabric & 59.17 & 57.32 & \bfseries 59.87 \\
			{} & {\scriptsize\textcolor{gray}{$\pm$0.96}} & {\scriptsize\textcolor{gray}{$\pm$2.15}} & {\scriptsize\textcolor{gray}{$\pm$2.60}} \\
			Miniboone & 74.10 & 76.23 & \bfseries 81.08 \\
			{} & {\scriptsize\textcolor{gray}{$\pm$2.19}} & {\scriptsize\textcolor{gray}{$\pm$1.51}} & {\scriptsize\textcolor{gray}{$\pm$1.42}} \\
			Wine & 34.46 & \bfseries 86.41 & 85.62 \\
			{} & {\scriptsize\textcolor{gray}{$\pm$5.23}} & {\scriptsize\textcolor{gray}{$\pm$1.25}} & {\scriptsize\textcolor{gray}{$\pm$1.08}} \\
			Yeast & 24.57 & \bfseries 32.38 & 30.82 \\
			{} & {\scriptsize\textcolor{gray}{$\pm$1.55}} & {\scriptsize\textcolor{gray}{$\pm$2.79}} & {\scriptsize\textcolor{gray}{$\pm$3.09}} \\
			\midrule
			FashionMNIST & 78.06 & 76.44 & \bfseries 81.22 \\
			{} & {\scriptsize\textcolor{gray}{$\pm$0.42}} & {\scriptsize\textcolor{gray}{$\pm$1.18}} & {\scriptsize\textcolor{gray}{$\pm$0.59}} \\
			Imagenette & 94.28 & 94.26 & \bfseries 94.80 \\
			{} & {\scriptsize\textcolor{gray}{$\pm$0.38}} & {\scriptsize\textcolor{gray}{$\pm$0.57}} & {\scriptsize\textcolor{gray}{$\pm$0.43}} \\
			MNIST & 82.71 & 81.12 & \bfseries 86.69 \\
			{} & {\scriptsize\textcolor{gray}{$\pm$0.51}} & {\scriptsize\textcolor{gray}{$\pm$0.98}} & {\scriptsize\textcolor{gray}{$\pm$0.62}} \\
			\midrule
			Avg Rank (Tabular) & 2.67 & 1.78 &  \textbf{1.56} \\
			Avg AUAC (Tabular) & 46.55 & 63.11 & \textbf{64.06} \\
			 & {\scriptsize\textcolor{gray}{$\pm$15.70}} & {\scriptsize\textcolor{gray}{$\pm$17.18}} & {\scriptsize\textcolor{gray}{$\pm$17.37}} \\
			
			Avg Rank (Image) & 2.00 & 3.00 & \textbf{1.00} \\
			Avg AUAC (Image) & 85.02 & 83.94 & \textbf{87.57} \\
			 & {\scriptsize\textcolor{gray}{$\pm$6.82}} & {\scriptsize\textcolor{gray}{$\pm$7.54}} & {\scriptsize\textcolor{gray}{$\pm$5.58}} \\
			\bottomrule
	\end{tabular}
\end{table}

\begin{figure}[p]
    \centering
    \renewcommand{\arraystretch}{1.2}
    \begin{tabular}{@{}ccc@{}}
        \multicolumn{3}{c}{\textsc{Tabular datasets}} \\[0.3em]
        \includegraphics[width=0.32\linewidth]{figures/zeroshot/tabular_bank_cardinality_f1.pdf} &
        \includegraphics[width=0.32\linewidth]{figures/zeroshot/tabular_california_housing_cardinality_f1.pdf} &
        \includegraphics[width=0.32\linewidth]{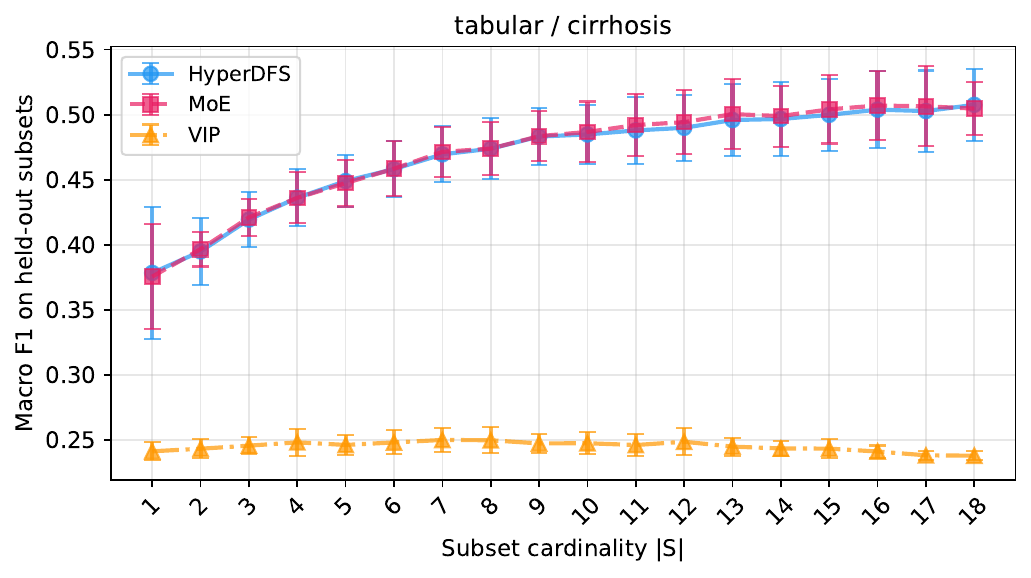} \\[-0.3em]
        {\small (g) Bank} & {\small (h) California} & {\small (i) Cirrhosis} \\[0.5em]
        \includegraphics[width=0.32\linewidth]{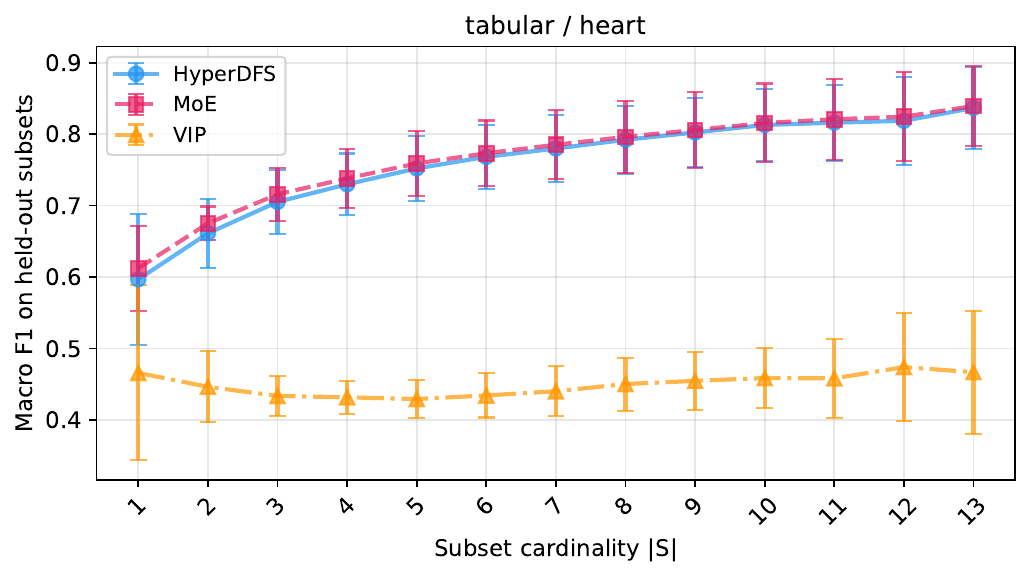} &
        \includegraphics[width=0.32\linewidth]{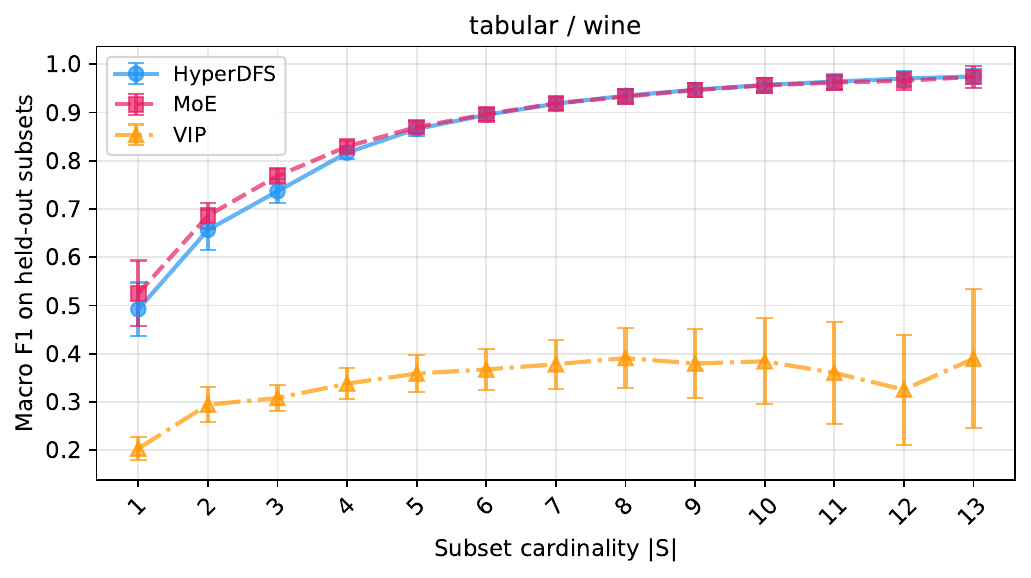} &
        \includegraphics[width=0.32\linewidth]{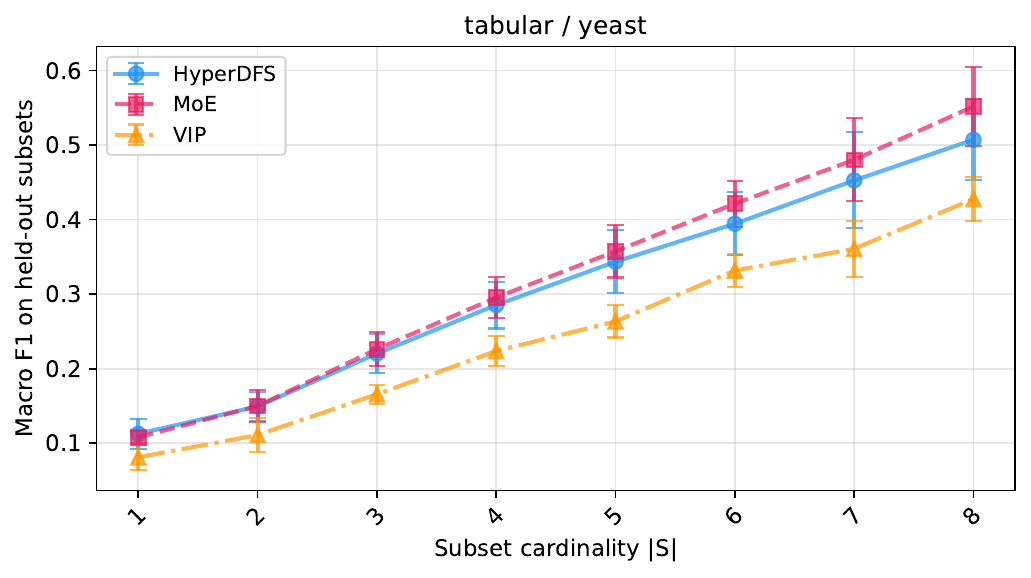} \\[-0.3em]
        {\small (j) Heart} & {\small (k) Wine} & {\small (l) Yeast} \\
    \end{tabular}
    \caption{F1-macro acquisition curves for all tabular datasets for unseen feature subsets in training. Each curve reports the mean F1-macro across 5-fold CV as a function of the number of acquired features, for the top-performing methods.}
    \label{fig:appendix_curves_zeroshot}
\end{figure}

Table~\ref{tab:zeroshot_auac_f1} reports the full per-dataset zero-shot AUAC-F1 results that underlie the aggregated report in the main text. VIP is the one clearly performing worse: on California, Cirrhosis, Heart, and Wine, VIP loses more than 22 absolute points to both MoE and the Hypernetwork, with the largest gap on Wine ($-51.95$ against MoE), whereas on Bank, Metabric, and Miniboone the gap shrinks to between 2 and 11 points. Diabetes is the only tabular dataset where VIP narrowly leads ($+0.33$ over \textsc{Hyper-DFS}, $+4.17$ over MoE), and Yeast is the one where all three methods cluster within a 7-point band on the lowest absolute scores in the benchmark. These cases suggest that single-model degradation under unseen subsets is most pronounced when the task admits distinct decision boundaries for different feature combinations, while on naturally hard datasets, the room for per-subset specialisation is small and the three approaches converge.

The two multi-model methods are close on tabular benchmarks: \textsc{Hyper-DFS} wins on four datasets (Bank, California, Metabric, Miniboone), MoE wins on four (Cirrhosis, Heart, Wine, Yeast), and the average rank (1.56 vs.\ 1.78) and average AUAC (64.06 vs.\ 63.11) place them within a fraction of a point of each other. The largest gap in either direction is on Miniboone ($+4.85$ for \textsc{Hyper-DFS}), which is the highest-dimensional tabular benchmark in the suite. We emphasise that this parity is achieved with a fundamentally different cost profile: MoE stores $K$ distinct sets of expert parameters, whereas \textsc{Hyper-DFS} is a single model where  all subset-specific classifiers are generated on demand from a shared hypernetwork. The image benchmarks show a more positive balance for VIP. However, \textsc{Hyper-DFS} still wins on all three datasets, with the largest gaps on MNIST ($+3.98$ over VIP) and FashionMNIST ($+3.16$ over VIP), while MoE underperforms VIP on every image benchmark. We can conclude that across both modalities, \textsc{Hyper-DFS} therefore matches or exceeds the strongest multi-model baseline while using the storage cost of a single model.

\subsection{Sequential acquisition examples}

\begin{figure}[t]
    \centering
    \begin{subfigure}[t]{0.48\linewidth}
        \centering
        \includegraphics[width=\linewidth]{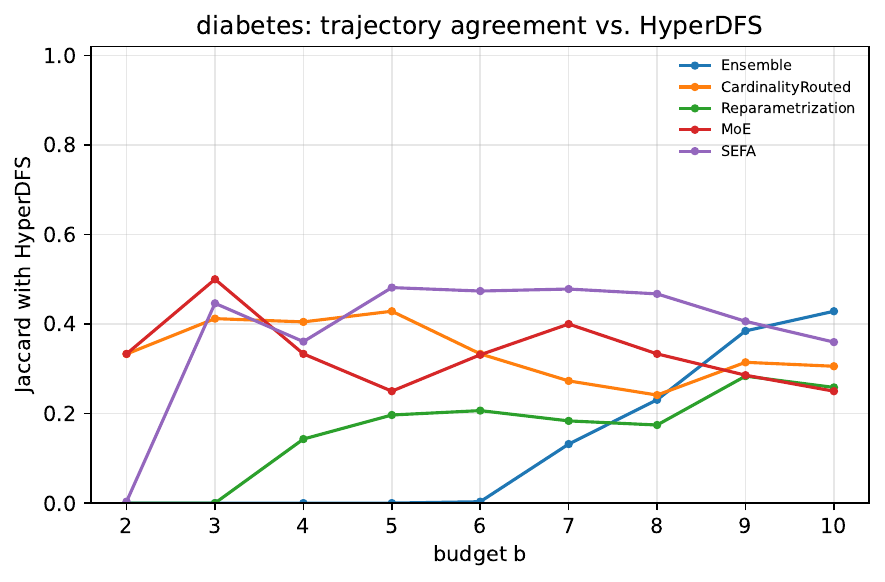}
        \label{fig:jaccard_budget_diabetes}
    \end{subfigure}
    \hfill
    \begin{subfigure}[t]{0.48\linewidth}
        \centering
        \includegraphics[width=\linewidth]{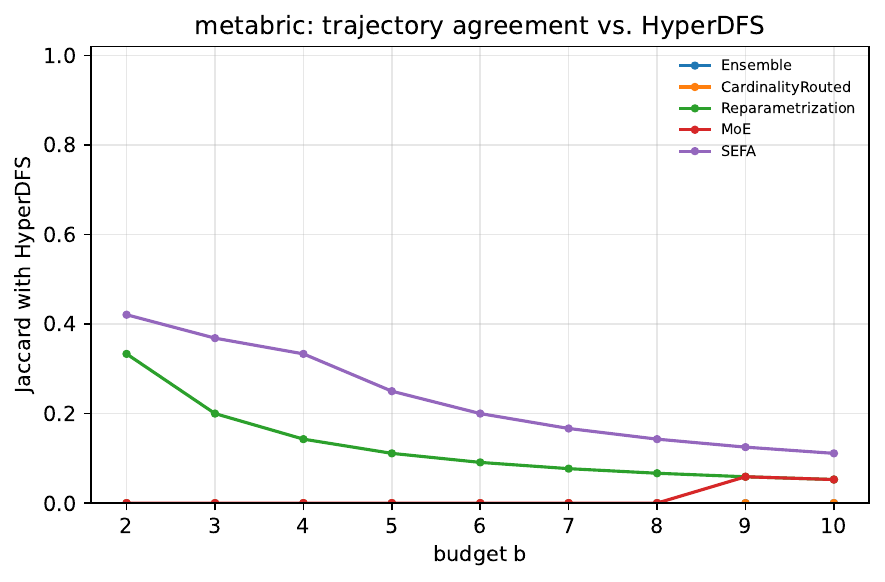}
    \end{subfigure}
    \caption{Evolution of Jaccard similarity of the features selected in each budget, comparing different DFS methods to \textsc{Hyper-DFS}.}
    \label{fig:comparison_trajectories_budget}
\end{figure}

\begin{figure}[t]
    \centering
    \begin{subfigure}[t]{0.48\linewidth}
        \centering
        \includegraphics[width=\linewidth]{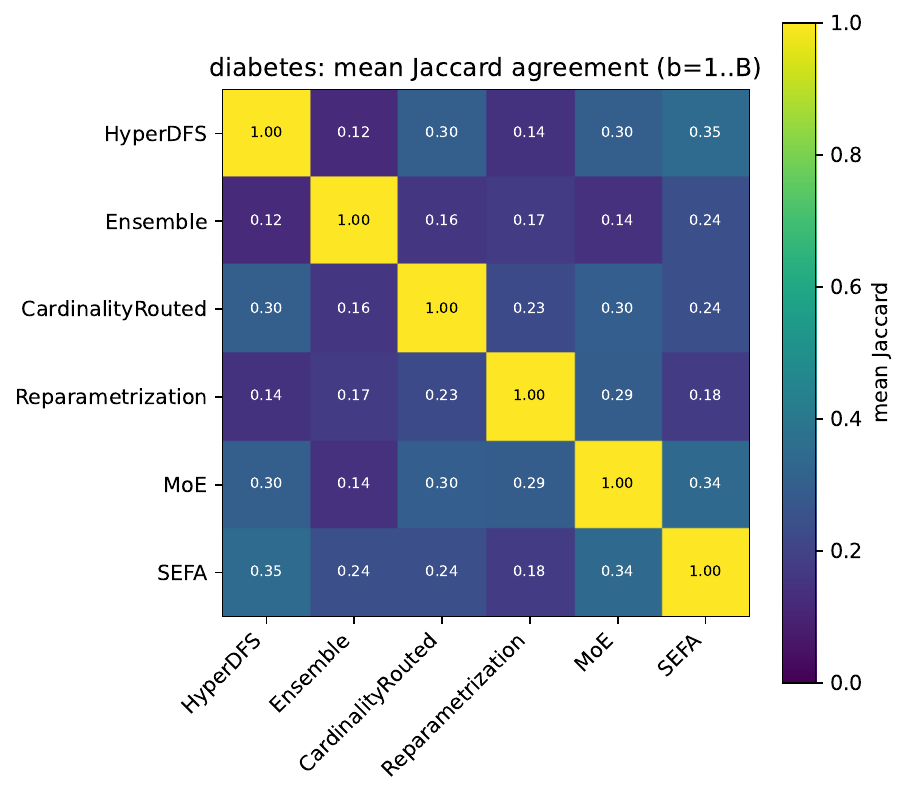}
        \caption{Diabetes}
        \label{fig:jaccard_heatmap_diabetes}
    \end{subfigure}
    \hfill
    \begin{subfigure}[t]{0.48\linewidth}
        \centering
        \includegraphics[width=\linewidth]{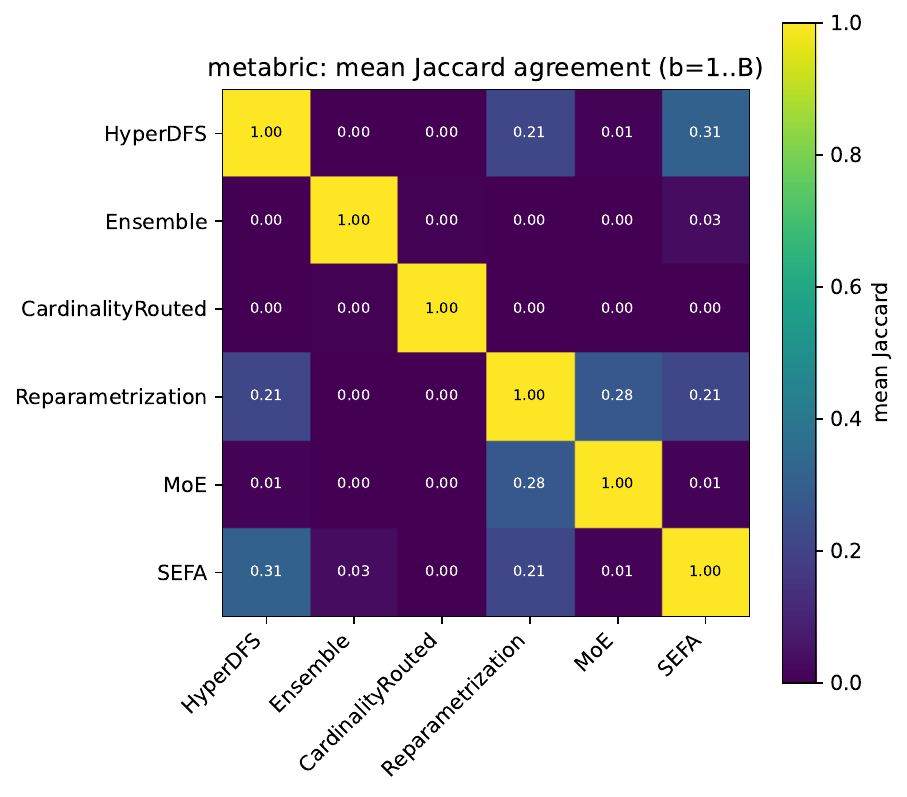}
        \caption{Metabric}
    \end{subfigure}
    \caption{Pairwise Jaccard similarity between feature subsets selected at different acquisition budgets for different DFS methods. Each cell $(i,j)$ represents the overlap between subsets selected at budgets $i$ and $j$, highlighting the stability of the acquisition policy across datasets.}
    \label{fig:jaccard_heatmaps}
\end{figure}

\begin{figure}[t]
    \centering
    \begin{subfigure}[t]{0.48\linewidth}
        \centering
        \includegraphics[width=\linewidth]{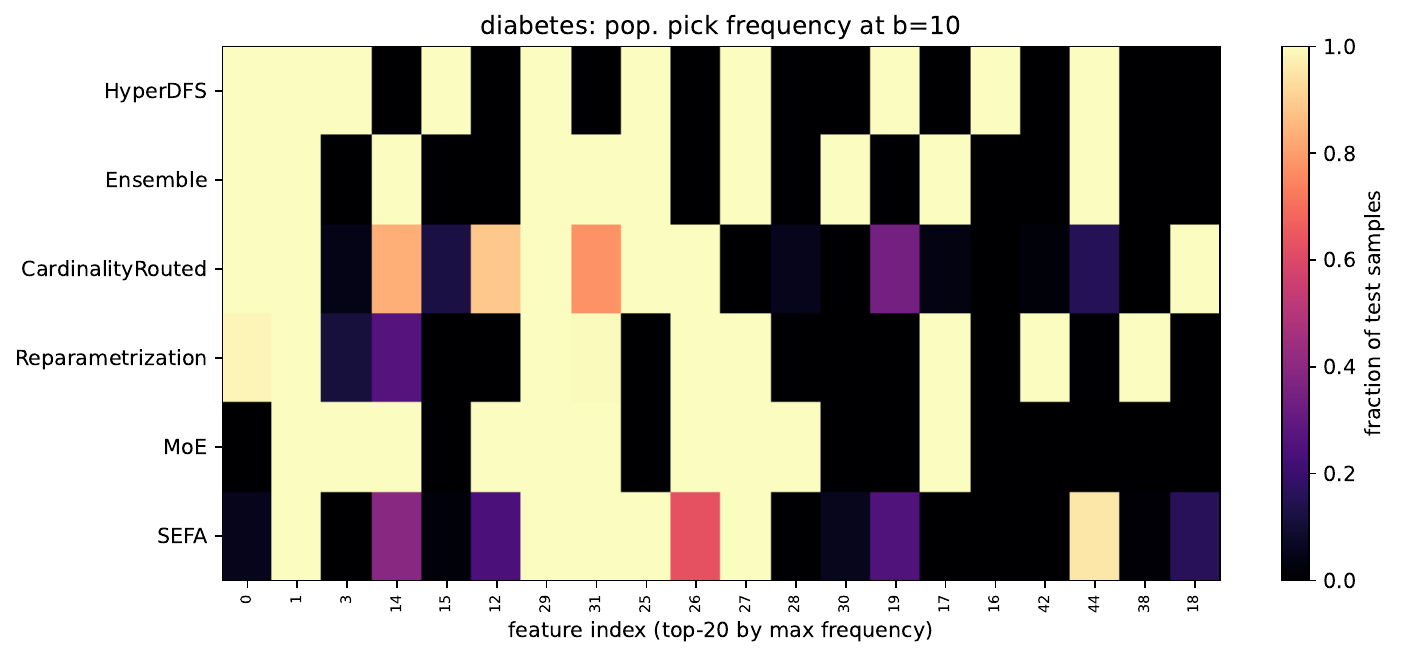}
        \caption{Diabetes}
        \label{fig:feature_freq_diabetes}
    \end{subfigure}
    \hfill
    \begin{subfigure}[t]{0.48\linewidth}
        \centering
        \includegraphics[width=\linewidth]{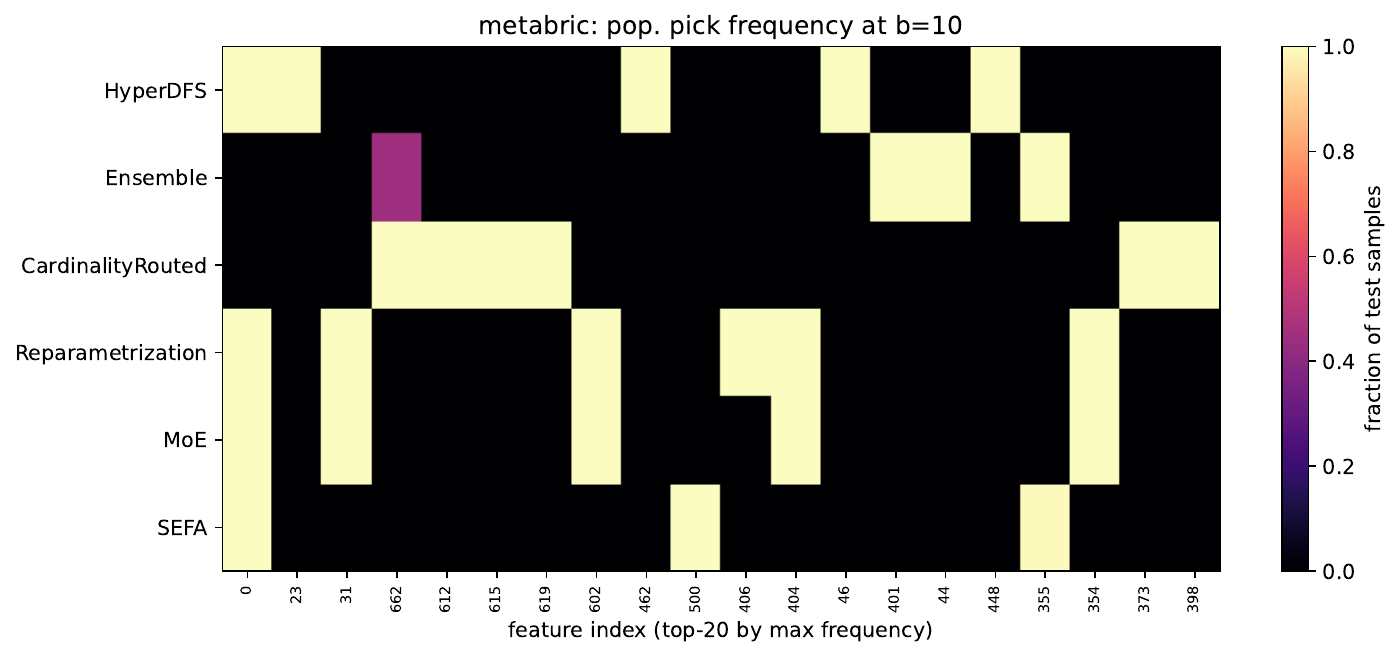}
        \caption{Metabric}
        \label{fig:feature_freq_metabric}
    \end{subfigure}
    \caption{Frequency of selection per sample of the top 20 most popular features selected overall by the corresponding methods for budget$=10$.}
    \label{fig:feature_frequency}
\end{figure}

Here, we show an example of the level of agreement in feature selection of the different DFS schemes tested. 
Figure \ref{fig:comparison_trajectories_budget} shows the Jaccard agreement for different budget levels of different DFS methods with \textsc{Hyper-DFS}. We can see there that the selection similarities in Metabric, where the number of features is very large, are significantly smaller than in diabetes, which is somewhat expected. There is no clear relationship between agreement with \textsc{Hyper-DFS} (or with any other method whatsoever) and performance, which indicates that many possible selections can give similar performance. This can explain the poor performance of DIME in datasets with large amounts of features, where it needs to estimate the CMI of all of them in order to make good decisions. Other methods just need to focus on which features make bigger performance increments, which is an easier task. 

A pairwise analysis of the feature similarities is also shown in Figure \ref{fig:jaccard_heatmaps}, where we can also see that in Metabric, sometimes the similarities are virtually 0 between feature selections. Although it is especially relevant to note that SEFA and Reparametrization (RePa), which are two other well-performing methods, have some significant overlap with \textsc{Hyper-DFS} as well as with each other. This suggests that there are some features that are key to reaching good decisions in some cases, no matter the model, even with such a large number of features. 

Figure \ref{fig:feature_frequency} shows the top 20 features most queried by the selectors for these two datasets as well. Here we can see that some methods indeed rely on the same features very often, while we also note significant divergences in Metabric as well. It is quite shocking that for Diabetes, the most popular feature is never queried by SEFA or MoE, which is another indication that models are learning different ways to get relevant information from different features as well, which might also be carrying redundant information. 

\subsection{Training times}

\begin{figure}
    \centering
    \includegraphics[width=0.6\linewidth]{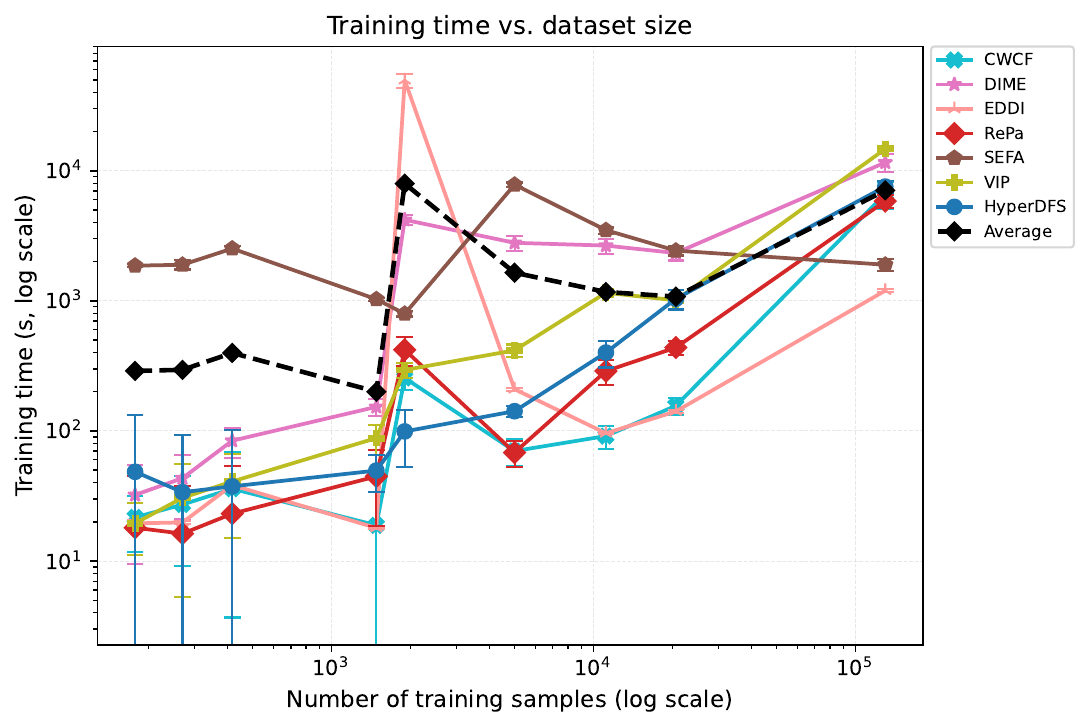}
    \caption{Running times for the different DFS algorithms tested in tabular datasets.}
    \label{fig:fit_time}
\end{figure}

Figure~\ref{fig:fit_time} reports wall-clock training times as a function of dataset size for \textsc{Hyper-DFS} and the single-model baselines using the same computational resources. On small datasets (fewer than ${\sim}500$ samples),  all methods are fast, with most completing within a minute. As dataset size grows, costs diverge considerably. SEFA is by far the most expensive method at medium and large scales, exceeding two hours on \textit{Diabetes}. DIME also becomes prohibitively slow at medium scale, exceeding 45 minutes on \textit{Bank} and \textit{California}. \textsc{Hyper-DFS} is below the average across all dataset sizes, and scales similarly to CWCF and Reparametrization at large $n$. On MiniBooNE ($n{=}130{,}064$), \textsc{Hyper-DFS} requires approximately the same training time as CWCF and Reparametrization (${\sim}$90--105 minutes), and is roughly $2.3\times$ faster than VIP.

\section{Reproducibility} \label{apx:implementation}

\subsection{Datasets Preprocessing Details} \label{apx:datasets_details}

  \begin{table}[t]                                                                                                                                                                            
      \centering
      \caption{Tabular benchmark datasets. $N$: samples, $M$: features, $C$: classes.}
      \label{tab:datasets}
      \begin{tabular}{llrrrc}
          \toprule
          Dataset & Domain & $N$ & $M$ & $C$ \\
          \midrule
          Wine~\citep{wine_109}                          & Scientific  &     178 &  13 &  3 \\
          Heart Disease~\citep{Detrano1989InternationalAO} & Medical   &     270 &  13 &  2 \\
          Cirrhosis~\citep{fleming2013counting}          & Medical     &     418 &  18 &  3 \\
          Yeast~\citep{yeast_110}                        & Biology     &   1{,}484 &   8 & 10 \\
          METABRIC~\citep{curtis2012genomic}             & Genomics    &   1{,}904 & 689 &  2 \\
          Bank Marketing~\citep{moro2014data}            & Financial   &  11{,}162 &  16 &  2 \\
          California Housing~\citep{pace1997sparse}      & Real estate &  20{,}640 &   9 &  2 \\
          Diabetes~\citep{strack2014impact}   & Medical     &   101766 &  45 &  3 \\
          MiniBooNE~\citep{roe2005boosted}               & Physics     & 130{,}064 &  50 &  2 \\
          \bottomrule
      \end{tabular}
  \end{table}

  \begin{table}[t]                                                      
      \centering
      \caption{Synthetic benchmark datasets. $M$: features, $C$: classes,                                                                                                                     
               $M^*$: number of features relevant to the Bayes-optimal predictor.
               All datasets use $N=10{,}000$ generated samples.}
      \label{tab:synthetic_datasets}
      \adjustbox{width=\linewidth}{
      \begin{tabular}{llcccl}
          \toprule
          Dataset & Reference & $M$ & $C$ & $M^*$ & Signal structure \\
          \midrule
          Sim1 & \citet{yoon2018invase} & 11 & 2 & 6 & Linear combination of fixed feature set \\
          Sim2 & \citet{yoon2018invase} & 11 & 2 & 6 & Nonlinear interactions within fixed feature set \\
          Sim3 & \citet{yoon2018invase} & 11 & 2 & 3--5 & Instance-wise gating; relevant set depends on $x_{10}$ \\
          Cube & \citet{kompella2016optimal} & 20 & 8 & 3 & Class-specific discriminative feature triplets \\
          \midrule
          Syn.\ Pairs & Ours & 12 & 2 & 6 & Product interactions; zero marginal signal per feature \\
          Proxy Sub.  & Ours & 10 & 2 & 5 & Heteroscedastic proxies; precision-weighted Bayes optimum \\
          \bottomrule
      \end{tabular}
      }
  \end{table}
                
  \paragraph{Synthetic datasets.}                           
  Table~\ref{tab:synthetic_datasets} summarises the six synthetic benchmarks. All datasets are generated with $N=10{,}000$ samples.
  \textbf{Sim1--3} and \textbf{Cube} follow the original protocols of
  \citet{yoon2018invase} and \citet{kompella2016optimal} respectively, and are included for comparability with the DFS literature.
  These benchmarks favour instance-wise feature selectors: once the relevant feature set is identified, the Bayes-optimal predictor is essentially the same function regardless of which subset is observed.

\textbf{Synergistic Pairs} ($M=12$) consists of six independent feature pairs, each drawn i.i.d.\ from $\mathcal{N}(0,1)$. Three pairs $(x_1, x_2)$, $(x_3, x_4)$, $(x_5, x_6)$ are signal pairs; the remaining three are pure noise. The label is generated as
\begin{equation}
    p(y=1 \mid \mathbf{x}) = \sigma\!\left(x_1 x_2 + x_3 x_4 + x_5 x_6\right),
\end{equation}
where $\sigma$ is the logistic function. Because $\mathbb{E}[x_i x_j \mid x_i] = x_i \cdot \mathbb{E}[x_j] = 0$ for independent features, every individual feature has zero mutual information with $y$. Signal is only accessible when both members of a pair are observed simultaneously, which requires the predictor to adapt its weights to the observed subset.

 \textbf{Proxy Substitution} ($M=10$) generates a shared latent variable $z \sim \mathcal{N}(0,1)$ observed through five heteroscedastic proxies $x_i = z + \varepsilon_i$, $\varepsilon_i \sim \mathcal{N}(0, \sigma_i^2)$ with $\boldsymbol{\sigma} = (0.1,\, 0.5,\, 1.0,\, 2.5,\, 5.0)$, and five independent noise features $x_6, \ldots, x_{10} \sim \mathcal{N}(0,1)$. Labels are $y = \mathbf{1}[z > 0]$. The Bayes-optimal linear predictor for any observed subset $S$ is the Gauss-Markov precision-weighted estimator
\begin{equation}
    \hat{z}(S) =
        \frac{\displaystyle\sum_{i \in S_p} x_i / \sigma_i^2}
             {1 + \displaystyle\sum_{i \in S_p} 1/\sigma_i^2},
    \qquad
    p(y=1 \mid \mathbf{x}_S) =
        \Phi\!\left(\hat{z}(S) \cdot
        \sqrt{1 + \textstyle\sum_{i \in S_p} \sigma_i^{-2}}\right),
\end{equation}
where $S_p = S \cap \{1, \ldots, 5\}$ and $\Phi$ is the standard normal cumulative distribution function. The optimal weight for each proxy depends on which other proxies are available, changing substantially with the observed subset and making fixed-weight single models systematically suboptimal.

  \paragraph{Tabular datasets.}
  Table~\ref{tab:datasets} summarises the nine real-world benchmarks.
  All datasets undergo the same preprocessing pipeline, applied
  independently within each cross-validation fold to prevent data leakage: numeric features are imputed at the training-fold median and standardised to zero mean and unit variance. Categorical features are imputed with the training-fold mode. For California Housing, the median house value target is binarised at its global median. For METABRIC, we dropped survival months and death from cancer fields to prevent target leakage, and also the patient id row.

\begin{table}[t]
    \centering
    \caption{Image benchmark datasets under patch-based sequential acquisition.
             $M$: number of patches (acquisition steps), $d$: patch feature
             dimension, $C$: classes. $^\dagger$ subsampled from the original size.}
    \label{tab:image_datasets}
    \begin{tabular}{llrrrrl}
        \toprule
        Dataset & Reference & $N_\text{train}$ & $N_\text{test}$ & $M$ & $d$ & $C$ \\
        \midrule
        MNIST         & \citet{lecun2002}       & 60{,}000 & 10{,}000 & 49 &  16 & 10 \\
        Fashion MNIST & \citet{xiao2017fashion} & 60{,}000 & 10{,}000 & 49 &  16 & 10 \\
        SVHN          & \citet{netzer2011svhn}  & 73{,}257 & 26{,}032 & 64 &  48 & 10 \\
        Imagenette    & \citet{howard2019}      &  9{,}469 &  3{,}925 & 25 & 512 & 10 \\
        PCam$^\dagger$          & \citet{veeling2018pcam} & 50{,}000 &     10{,}000   & 64 &  144 &  2 \\
        \bottomrule
    \end{tabular}
\end{table}

  \paragraph{Image datasets.}
Table~\ref{tab:image_datasets} summarises the five image benchmarks.
All datasets are evaluated under the same protocol: each image is divided into a spatial grid of non-overlapping patches, and the budget $b$ controls how many patches are revealed at test time. The acquisition order is determined by the model's policy and all remaining patches are masked (zeroed) before prediction.

MNIST and Fashion MNIST: each $28\times28$ grayscale image is divided into a $7\times7$ grid of $4\times4$ pixel patches, yielding $M=49$ patches with a raw pixel representation of dimension $d=16$. Pixel values are normalised to $[0,1]$.

SVHN~\citep{netzer2011svhn} consists of $32\times32$ RGB street-view digit images. Each image is divided into an $8\times8$ grid of $4\times4$ pixel patches, yielding $M=64$ patches of dimension $d=48$ (raw RGB pixels, normalised to $[0,1]$). For this dataset, in \textsc{Hyper-DFS}, we use a fix the number of knowledge statuses per batch in the predictor training phase to $3$.

Imagenette is a 10-class subset of ImageNet comprising higher-resolution natural images. Raw pixels are not used directly; instead, images are resized to $160$px and centre-cropped to $128\times128$ before being passed through a frozen ResNet-18 backbone pretrained on ImageNet~\citep{he2016deep}. The spatial feature map from the final residual block is pooled to a $5\times5$ grid and reshaped to $M=25$ patch tokens of dimension $d=512$. The backbone weights are fixed throughout all experiments; only the DFS model heads are trained. Note that the frozen backbone might have been trained on some of the test images of Imagenette, so this comparison should be seen as a controlled comparison between different DFS methods and not as an absolute performance study.

PCam~\citep{veeling2018pcam} is a binary histopathology dataset of $96\times96$ RGB lymph node patches, labelled by the presence of metastatic tissue in their centre $32\times32$ region. We divide the image into an $8\times8$ grid of $12\times12$ pixel patches, yielding $M=64$ patches of dimension $d=144$. The label in this dataset is determined by the central region, which the acquisition policy must learn to prioritise. Due to the scale of the dataset, we use a stratified subsample of 50{,}000 images and evaluate via $k$-fold cross-validation (no fixed test split).

\subsection{Model training details}

\paragraph{About the uniform sampling.} Eq.~\eqref{eq:single_model_obj} and Eq.~\eqref{eq:hyperdfs_ob_obj} use a uniform distribution over $S$. Sampling $S$ uniformly over $2^{[M]}$ via independent Bernoulli$(1/2)$ on each variable concentrates the cardinality $|S|$ around $M/2$, leaving small and large subsets severely under-represented. To obtain balanced coverage across cardinalities, we instead sample the cardinality first and then a set of that cardinality: $k \sim \mathcal{U}(\{1,\dots,M\})$, $S \sim \mathcal{U}\!\left(\binom{[M]}{k}\right)$.

\paragraph{\textsc{Hyper-DFS}.} Table~\ref{tab:hparams} lists the fixed hyperparameter configuration used in all \textsc{Hyper-DFS} experiments. The compressor is only used for image datasets to reduce the input size of the hypernetwork. We use a linear LR warmup that ramps the learning rate from $\eta/100$ to $\eta$ over the first $5$ epochs, after which a cosine schedule decays it back to $\eta/100$ over the remaining epochs.

  \begin{table}[t]
      \centering
      \caption{\textsc{Hyper-DFS} training hyperparameters.}
      \label{tab:hparams}
      \begin{tabular}{ll}
          \toprule
          Hyperparameter & Value\\
          \midrule
          Optimiser         & Adam                               \\
          Learning rate & $10^{-2}$                       \\
          Weight decay      & $10^{-4}$                          \\
          Gradient clip norm & $5.0$                             \\
          LR schedule       & Cosine ($\eta_\mathrm{min}{=}0.01\eta$) \\
          Max epochs        & 200                                \\
          Early stopping patience & 30 \\
          Validation fraction & 10\% with 5-fold evaluation                             \\
          \midrule
          Noise $\sigma$ for training    & 0.20                               \\
          $\lambda_\mathrm{scale}$ & 0.1                         \\
          Scale warm-up epochs & 50                              \\
          $\lambda_\mathrm{collapse}$ & 0.01                     \\
          \midrule
          Primary mid layer size / Number of layers  & 64 / 2                                 \\
          Hypernetwork mid layer size /Number of layers & 128 /2  \\
          Compressor output dim (only on image datasets) & 16 \\
          \bottomrule
      \end{tabular}
  \end{table}

  \paragraph{Multi-Model Baselines} \textbf{Ensemble} averages the models resulting from different seed initialisations ($K=5$). \textbf{CardinalityRouted} trains two specialist MLPs with hidden layers $(64, 32)$, one for subsets of cardinality $|S| \leq \lfloor M/2 \rfloor$ and one for $|S| > \lfloor M/2 \rfloor$.  Each specialist is trained only on its assigned cardinality range; at inference, the appropriate specialist is selected by hard routing. \textbf{MoE} uses $K=5$ expert MLPs (hidden dim 64) with a router using a routing embedding dimension of 64. A load-balancing auxiliary loss ($\lambda_\text{balance}{=}0.01$) encourages uniform expert utilisation during training.

  \paragraph{Literature Single-Model Baselines} For all literature baselines, we use the original published implementations when available. Hyperparameters are set to the author-recommended values from the respective papers.

  \paragraph{Computational resources} Experiments were carried out on a local workstation (AMD Ryzen 5 5600X, 6 cores / 12 threads, paired with a single NVIDIA RTX 5050 8GB) running Ubuntu and a GPU server with compute nodes with a NVIDIA GTX 1080 Ti. The full synthetic protocol and the tabular benchmark sweep (7 datasets x all baselines x multiple seeds and feature budgets) together amount to on the order of a few hundred single-GPU hours. Individual jobs typically complete in minutes to an hour.
  
\newpage
\end{document}